\documentclass[twoside]{article}

\usepackage[accepted]{aistats2025}
\usepackage[square]{natbib}

\usepackage{amsmath}
\usepackage{amsthm}
\usepackage{amssymb}
\usepackage{bm}
\usepackage{dsfont}
\usepackage{caption}
\usepackage{csquotes}
\usepackage{mathtools}
\usepackage{subfigure}
\usepackage{url}
\usepackage{tikz}
\definecolor{red}{HTML}{F44336}
\definecolor{green}{HTML}{4CAF50}
\definecolor{yellow}{HTML}{FFEE58}
\definecolor{blue}{HTML}{0D47A1}
\usepackage[colorlinks,linkcolor=blue,citecolor=blue,urlcolor=green]{hyperref}

\newtheorem{theorem}{Theorem}
\newtheorem{proposition}{Proposition}

\newtheorem{lemma}{Lemma}

\theoremstyle{definition}
\newtheorem{definition}{Definition}
\newtheorem{remark}{Remark}

\newcommand{\0}{\bm{0}}

\renewcommand{\a}{\bm{a}}
\renewcommand{\b}{\bm{b}}

\newcommand{\x}{\bm{x}}
\newcommand{\y}{\bm{y}}
\newcommand{\z}{\bm{z}}
\newcommand{\g}{\bm{g}}

\newcommand{\I}{\bm{I}}
\newcommand{\W}{\bm{W}}
\newcommand{\X}{\bm{X}}
\newcommand{\Y}{\bm{Y}}
\newcommand{\Z}{\bm{Z}}
\newcommand{\G}{\bm{G}}
\newcommand{\bdel}{\boldsymbol{\delta}}

\newcommand{\beps}{\boldsymbol{\varepsilon}}
\newcommand{\bmu}{\boldsymbol{\mu}}

\newcommand{\bSigma}{\boldsymbol{\Sigma}}

\newcommand{\EE}{\mathbb{E}}

\newcommand{\PP}{\mathbb{P}}
\newcommand{\RR}{\mathbb{R}}

\newcommand{\bbmone}{\mathds{1}}

\newcommand{\cA}{\mathcal{A}}
\newcommand{\cE}{\mathcal{E}}
\newcommand{\cF}{\mathcal{F}}
\newcommand{\cH}{\mathcal{H}}

\newcommand{\cM}{\mathcal{M}}
\newcommand{\cN}{\mathcal{N}}
\newcommand{\cP}{\mathcal{P}}

\newcommand{\abs}[1]{\lvert #1 \rvert}
\newcommand{\inner}[2]{\langle #1, #2 \rangle}
\newcommand{\norm}[1]{\lVert #1 \rVert}

\newcommand{\RIG}{\mathrm{RIG}}
\newcommand{\NIRIG}{\mathrm{NIRIG}}
\newcommand{\Bern}{\mathrm{Bern}}
\newcommand{\Binom}{\mathrm{Binom}}
\newcommand{\TV}{\mathrm{TV}}
\newcommand{\KL}{\mathrm{KL}}
\renewcommand{\vec}{\mathrm{vec}}

\DeclareMathOperator*{\argmax}{arg\,max}
\DeclareMathOperator*{\argmin}{arg\,min}

\begin{document}



\twocolumn[

\aistatstitle{Perfect Recovery for Random Geometric Graph Matching with Shallow Graph Neural Networks}

\aistatsauthor{ Suqi Liu \And Morgane Austern }

\aistatsaddress{ Harvard University \And Harvard University } ]

\begin{abstract}%
We study the graph matching problem in the presence of vertex feature information using shallow graph neural networks. Specifically, given two graphs that are independent perturbations of a single random geometric graph with sparse binary features, the task is to recover an unknown one-to-one mapping between the vertices of the two graphs. We show under certain conditions on the sparsity and noise level of the feature vectors, a carefully designed two-layer graph neural network can, with high probability, recover the correct mapping between the vertices with the help of the graph structure. Additionally, we prove that our condition on the noise parameter is tight up to logarithmic factors. Finally, we compare the performance of the graph neural network to directly solving an assignment problem using the noisy vertex features and demonstrate that when the noise level is at least constant, this direct matching fails to achieve perfect recovery, whereas the graph neural network can tolerate noise levels growing as fast as a power of the size of the graph. Our theoretical findings are further supported by numerical studies as well as real-world data experiments.
\end{abstract}

\section{INTRODUCTION}
Graph neural networks~(GNNs)~\citep{kipf2016semi} have seen broad application
in many important domains involving graph-structured data since their inception,
including social networks~\citep{hamilton2017inductive},
computational biology~\citep{fan2019graph},
chemistry~\citep{gilmer2017neural},
and knowledge graphs~\citep{schlichtkrull2018modeling}.
A standard graph neural network model, also known as a message passing neural network,
consists of traditional multilayer perceptrons~(MLPs) injected with
a message passing step that aggregates the hidden representation from
neighboring vertices in the graph.
Although the message passing idea seems rather simple,
GNNs have achieved wide success in various tasks,
including node classification, link prediction, and learning graph properties
(see \cite{zhou2020graph} for a recent survey and references therein).

Despite their wide popularity and dominant performance in graph-based tasks,
the theoretical understanding of GNNs is only emerging
in recent years~\citep{jegelka2022gnn}.
Most of the works focus on traditional learning theory aspects
such as representation power~\citep{loukas2020graph}
and generalization~\citep{scarselli2018vapnik}.
A few of them touch on more classic graph-theoretic problems such as
graph isomorphism~\citep{morris2019weisfeiler}
and graph properties~\citep{garg2020generalization}.
Recently there has been a growing literature studying GNNs
using graphons~\citep{ruiz2021graphon,ruiz2023transferability,chung2024statistical}.
However, there is still a divide between the theory of GNNs
and how they are used in the real world.
The benefits and limitations of GNNs in many scenarios
to which they are commonly applied remain enigmatic.
Our goal is to expand the theoretical understanding by investigating
a common use case of GNNs.

To this end, we focus on the problem of aligning two geometric graphs 
together with noisy observations of their vertex features.
This simplified setup resembles several practical situations.
For example, suppose that we have access to two different
social networks that are actually built around the same group of people,
together with inaccurate features of each node,
while the node correspondence is not known to us
(think of Instagram and Facebook).
The task is to recover the underlying node correspondence from the networks
and node features.
We remark that this seemingly simplistic graph alignment task
has numerous instantiations, such as
cross-lingual knowledge graph alignment~\citep{wang2018cross} and
molecular network comparison~\citep{sharan2006modeling}.
%and indeed, GNNs achieve the state-of-the-art performance~\cite{zeng2021comprehensive}.
Traditional methods include structure-based and iteration-based approaches
(see \cite[Section~3]{zeng2021comprehensive} for a comprehensive survey).
More recently, entity alignment based on representation learning, 
including the application of GNNs, has become mainstream~\citep{zeng2021comprehensive}.

The purpose of this work is not to add new methods to this already abundant literature,
but instead to examine the performance of GNNs for the alignment problem
through the lens of probability theory.
Specifically, suppose that we observe two incomplete copies, $G$ and $G'$,
of the same random geometric graph with sparse binary features.
Additionally, for each graph, we also observe their vertex features perturbed
by independent Gaussian noise.
The goal is to match each vertex of $G$ with a vertex of $G'$.
In this work, we characterize how GNNs can facilitate this task leveraging both the noisy features
and the graph structure.

\subsection{Contributions}
Our main contribution is to analyze the performance of GNNs for graph alignment
tasks on a random geometric graph model and theoretically prove the benefit of
message passing in the presence of vertex features.
Specifically, the contribution is threefold:
\begin{enumerate}
\itemsep0em
\item We propose a random geometric graph model that generalizes many existing
models and closely resembles real-life settings, allowing for the formal study of
the graph alignment problem with the presence of vertex features.
\item We prove that in certain parameter regimes of the model,
perfect recovery is possible with a specially-designed two-layer graph neural
network.
We also show that the dependence on the noise parameter is tight up to
logarithmic factors.
\item Meanwhile, we demonstrate that
directly aligning the vertices with noisy features fails to recover
the vertex correspondence perfectly under certain conditions,
while it remains possible with the help of the graph neural network.
\end{enumerate}

\subsection{Limitations} \label{sc:limit}
Although the current work is theoretical in nature and aims to explain
the effectiveness of GNNs with probability tools,
we discuss the limitations from both theoretical
and application perspectives:
\begin{enumerate}
\itemsep0em
\item The random geometric graph model studied in this paper is still rather
idealistic,
even though we attempt to incorporate intuitions
from real-world applications as much as possible.
\item Only perfect recovery is considered in the current work for clarity of
the theoretical message,
while in practice, other notions of alignment performance could also be relevant.
\item In the same way classical neural network literature has been able to gain
valuable insights,
we learn some interesting phenomena of GNNs from studying this simple setting.
Notably we can highlight the role of message passing and nonlinearity,
which other works fail to do.
\end{enumerate}

\subsection{Related work}
Network alignment problem is usually approached through graph representation
learning methods in modern machine learning literature~\citep{yan2021bright}.
Earlier works focus on learning the low-dimensional vector representation of the entity
by simple similarity metrics~\citep{bordes2013translating}
or probabilistic models~\citep{grover2016node2vec}.
Later on, many variants of GNNs have been successfully applied,
achieving state-of-the-art performance.
Listing all relevant papers would be impossible,
so instead, we direct interested readers to
the survey article~\citep{zeng2021comprehensive}.
However, most of the advances are made from purely application-based considerations
and lack theoretical justification.

On the other hand, random graph matching has attracted considerable
theoretical investigation in recent years~\citep{ding2021efficient,mao2021random,wu2022settling,racz2023matching,fan2023spectralI,fan2023spectralII}.
However, the majority of the literature focuses on the correlated Erd\H{o}s--R\'enyi model
where two samples are independently created from the same Erd\H{o}s--R\'enyi graph
through edge-resampling.
This method of introducing noise is similar to our subsampling process of the geometric graph.
GNNs have recently been introduced to this line of research by~\citet{yu2023seedgnn},
which studies seeded graph matching with a carefully-designed architecture.
The study of matching random graphs with latent geometric structure has emerged only
in the past few years.
The correlated stochastic block model introduced by~\citet{racz2021correlated}
can be viewed as an intersection graph in one dimension.
A later work~\citep{racz2023matching} studies the so-called $k$-core estimator
for a wide range of inhomogeneous random graphs including random geometric graphs.
A recent paper by~\citet{wang2022random} considers matching random geometric graphs,
but they assume that all pairwise dot products or Euclidean distances are observed
under the Gaussian setting.
One significant difference that sets our model apart from previous literature is that
we consider a novel setting of sparse binary features with noisy observation.
Nevertheless,
we would like to emphasize that the main goal of our work is not to solve the graph
matching problem optimally but rather to show that the effectiveness of
GNNs in practice is indeed supported by theoretical analysis.

Due to the existence of vertex features,
our work is also related to matching two random point clouds
which has a long history in probability and combinatorics~\citep{ajtai1984optimal}.
More recently, \citet{kunisky2022strong} studied matching recovery for Gaussian perturbed
Gaussian vectors in various regimes by solving an assignment problem.
Notably, our feature generating process is similar to that investigated
by the authors~\citep{kunisky2022strong}.
Our work builds upon the same approach of solving an assignment problem
but shows theoretically that
with the help of the graph structure,
modern machine learning methods are capable of achieving much better recovery guarantees.
Again, we do not attempt to provide better recovery bounds on matching random points,
but rather to supply a theoretical understanding of how graphs aid in
the alignment problem.

The last literature we should mention pertains to learning guarantees for GNNs
trained on random graphs generated according to a graphon.
Stability and transferability of certain untrained GNNs have been established
in~\citep{ruiz2021graphon,maskey2023transferability,ruiz2023transferability,keriven2020convergence}.
Generalization capabilities of GNNs~\citep{maskey2022generalization,esser2021learning}
and the capacity of GNNs to distinguish different graphons~\citep{magner2020power}
have also been studied.
In a different direction, \citet{kawamoto2018mean,lu2021learning} considered the performance of
GNNs for community detection respectively through heuristic mean-field approximations
and formally for two community stochastic block models (SBMs) when the GNN is trained
via coordinate descent.
\citet{baranwal2021graph} studied node classification for contextual stochastic block models (CSBMs)
and showed that an oracle GNN can significantly boost the performance of linear classifiers.
\citet{wang2024optimal} investigated both a spectral algorithm and graph convolutional networks (GCNs)
for node classification in CSBMs.
\citet{duranthon2024optimal} derived a belief-propagation-based algorithm
for the same task and showed a considerable gap between the accuracy reached
by the proposed algorithm and the existing GNN architectures.
Finally, \citet{chung2024statistical} have obtained guarantees for linear GNNs
for the edge prediction task.
Our goal is significantly different from all of these aforementioned works
as we aim to establish the performance of GNNs for graph alignment tasks.

\subsection{Notations}
We use bold uppercase letters to denote matrices
and the corresponding lowercase letters to denote their rows.
For example, $\X \in \RR^{n \times d}$ is an $n$-by-$d$ real matrix
with row vectors $(\x_1, \ldots, \x_n)$.
For a vector $\x$, $\x(k)$ stands for its $k$th entry.
The $p$-norm of a vector is denoted by $\norm{\cdot}_p$ and it defaults to $2$-norm,
or the Euclidean norm, when $p$ is not specified.
The inner product between two $d$-dimensional vectors $\x$ and $\y$ is written as
$\inner{\x}{\y} \coloneqq \sum_{k=1}^d \x(k)\y(k)$.
The notation $\abs{\cdot}$ is generally overloaded.
When it is applied to a number $a$, $\abs{a}$ is the absolute value of $a$.
And when applied to a set $S$, $\abs{S}$ denotes its cardinality.
We use $[n] \coloneqq \{1, \ldots, n\}$ to denote the set of all natural numbers up to $n$.
For a set $S$, $S^c$ stands for its complement.
Gaussian (normal) distribution with mean $\mu$ and variance $\sigma^2$ is written as
$\cN(\mu, \sigma^2)$ and correspondingly the $d$-dimensional Gaussian
with mean vector $\bmu$ and covariance matrix $\bSigma$ is written
as $\cN(\bmu, \bSigma)$.
$\Phi(x), x \in \RR$ is the distribution function of $\cN(0, 1)$,
and $\Phi^c(x) \coloneqq 1 - \Phi(x)$ represents the upper tail.
We also make use of the standard big O notation:
For positive functions $f(x)$ and $g(x)$,
$f(x) \lesssim g(x)$ or $f(x) = O(g(x))$ if $f(x) \le C g(x)$ for a constant $C > 0$
independent of $x$;
$f(x) \gtrsim g(x)$ or $f(x) = \Omega(g(x))$ if $g(x) \lesssim f(x)$.
We write $f(x) \asymp g(x)$ if $f(x) = O(g(x))$ and $g(x) = O(f(x))$.
We denote $f(x) \ll g(x)$, $g(x) \gg f(x)$, or $f(x) = o(g(x))$
if $\lim_{x \to \infty} f(x)/g(x) = 0$.

\section{PROBLEM DEFINITION} \label{sc:model}
Our model is a generalization of the random intersection graph
within the family of random geometric graphs.
We first introduce several notations necessary for defining the model.
A graph $G$ on $n$ vertices is denoted by a pair $G = (\X, E)$
where $\X = (\x_1, \ldots, \x_n)$ is the list of vertex features
and $E$ is the set of undirected edges.
That is, vertices $i$ and $j$ are connected by an undirected edge
if and only if $\{i, j\} \in E$.

\begin{definition}[Random intersection graph] \label{df:rig}
The random intersection graph $G_0 = (\X, E_0)$ is defined as follows.
A $d$-dimensional binary feature vector $\x_i \in \{0, 1\}^d$ is associated with each vertex $i$.
Given a sparsity parameter $s \le d$,
the entries of $\x_i$ follow an independent Bernoulli distribution with parameter $s/d$,
i.e., $\PP(\x_i(j) = 1) = s/d$ and $\PP(\x_i(j) = 0) = 1 - s/d$ independently.
For each pair of vertices $i \ne j \in [n]$,
$\{i, j\} \in E_0$ if and only if
\begin{equation*}
\inner{\x_i}{\x_j} \ge t
\end{equation*}
for a fixed threshold $t \ge 1$.
We denote the graph by $\RIG(n, d, s, t)$.
\end{definition}

Traditionally, a random intersection graph is defined by $n$ random subsets
of total $d$ elements,
and sets are connected if their intersection
is nonempty~\citep{singer1996random}.
It is clear that this corresponds to the case $t = 1$
in our more general definition of RIG.
Written in the form of Definition~\ref{df:rig},
it is also clear that
RIG is a special case of random inner (dot) product graphs~\citep{young2007random}
when features are Bernoulli random variables.
In contrast to the Erd\H{o}s--R\'enyi model,
edges in RIG are regulated by the underlying geometric space,
which is in line with other random geometric graphs~\citep{penrose2003random}.

Suppose that for a graph $G_0$, we do not have direct access to it.
Instead, we are given two noisy and incomplete copies of it,
where we observe perturbed feature vectors
and retain only a subset of the edges.
The procedure is designed to mimic the data patterns in real-world networks.
\begin{definition}[Noisy and incomplete RIG] \label{df:nirig}
Given a graph $G_0 = (\X, E_0) \sim \RIG(n, d, s, t)$,
we assume that the observed graph $G = (\Y, E)$
is created by the following process.
\begin{enumerate}
\itemsep0em
\item For each vertex $i \in [n]$, the observed feature vector is given by
\begin{equation*}
\y_i = \x_i + \beps_i
\end{equation*}
where $\beps_i \sim \cN(\0, \sigma^2\I_d)$ is independent noise
for some noise parameter $\sigma \ge 0$.
\item For each pair of vertices $i \ne j \in [n]$,
an edge $\{i, j\} \in E$  between them is observed with probability $q$ only if $\{i, j\} \in E_0$.
In other words,
\begin{equation*}
\PP(\{i, j\} \in E \mid \{i, j\} \in E_0) = q
\end{equation*}
and $\{i, j\} \notin E$ otherwise.
\end{enumerate}
The graph is denoted by $\NIRIG(\X, \sigma, q)$.
\end{definition}

We create a pair of correlated NIRIGs as follows.
First, we generate a truth graph $G_0 = (\X, E_0) \sim \RIG(n, d, s, t)$.
Then, we obtain two independent samples of $\NIRIG(\X, \sigma, q)$, $G$ and $G'$,
from the same $G_0$,
along with a vertex permutation between them.
In other words, given $\X$,
\begin{equation*}
G = (\Y, E) \sim \NIRIG(\X, \sigma, q)
\end{equation*}
and
\begin{equation*}
G' = (\Y', E') \sim \NIRIG(\pi^*(\X), \sigma, q)
\end{equation*}
independently for an unknown permutation $\pi^*$.
Our goal is to recover the true permutation $\pi^*$ from the observations
$G = (\Y, E)$ and $G' = (\Y', E')$.

\section{MAIN RESULTS} \label{sc:main}
The vanilla graph neural network (also called a graph convolutional network
or message passing neural network)
iteratively applies the following propagation rule (from layer $l$ to $l+1$)
to each vertex $i$ in $G = (\X, E)$:
\begin{equation*}
\x_i^{l+1} = \eta\biggl(\frac{1}{\abs{N_i}}\sum_{j \in N_i} \W^l \x_j^{l}\biggr)
\end{equation*}
where $\eta$ is a nonlinear activation function,
$N_i \coloneqq \{j \in [n]: \{i, j\} \in E\}$ is the neighborhood of $i$,
and $\W^l$ is the trainable weight matrix.

We employ a specially-designed two-layer
graph neural network to find the matching between $G = (\Y, E)$ and $G' = (\Y', E')$.
Define the neighborhood of vertex $i$ in $G$ and $G'$ respectively as
\begin{equation*}
N_i \coloneqq \{j \in [n]: \{i, j\} \in E\}
\end{equation*}
and
\begin{equation*}
N_i' \coloneqq \{j \in [n]: \{i, j\} \in E'\}.
\end{equation*}
We apply a message passing layer to the observations $\Y$ and $\Y'$
followed by a thresholding function
\begin{equation*}
\eta(u) = \bbmone\biggl\{u \ge \frac{t}{2s}\biggr\}.
\end{equation*}
Hence for each vertex $i \in [n]$ the hidden units of
the two-layer graph neural network
from the input graphs $G$ and $G'$ are respectively
\begin{equation*}
\z_i = \eta\biggl(\frac{1}{\abs{N_i}}\sum_{j \in N_i} \y_j\biggr)
\quad\textrm{and}\quad
\z_i' = \eta\biggl(\frac{1}{\abs{N_i'}}\sum_{j \in N_i'} \y_j'\biggr).
\end{equation*}
These values are finally used to find the matching $\pi^\star$.
This is done by solving an assignment problem between $\Z$ and $\Z'$:
\begin{equation} \label{eq:mp_align}
\begin{split}
\hat{\pi} &= \argmin_{\pi \in \cP(n)} \biggl(\ell(\pi) \coloneqq \sum_{i=1}^n \norm{\z_i - \z_{\pi(i)}'}^2\biggr)\\
&= \argmax_{\pi \in \cP(n)} \sum_{i=1}^n \inner{\z_i}{\z_{\pi(i)}'}
\end{split}
\end{equation}
where $\cP(n)$ is the permutation group on $[n]$.
This problem can be solved in polynomial time with combinatorial optimization algorithms such as
the Hungarian algorithm~\citep{kuhn1955hungarian}.

The RIG is parametrized
by the sparsity parameter $s$ of the underlying features.
However,
it can be directly translated to properties of the underlying graph $G_0$.
When $n$, $d$, and $t$ are fixed,
the edge density of the graph is indeed determined by $s$ in a nontrivial way.
To make this dependence explicit,
we define another parameter $m \in [0, n]$
such that
\begin{equation} \label{eq:s}
s = \sqrt{td\biggl(\frac{m}{n}\biggr)^{1/t}}.
\end{equation}
It will be demonstrated by a later lemma (Lemma~\ref{lm:N_i} in the appendix) that $m$
roughly characterizes the average degree in $\RIG(n, d, s, t)$ when $t$ is constant.
For simplicity,
we keep $t \ge 1$ fixed in our discussion.
This notably implies,
using \eqref{eq:s},
that $s \lesssim \sqrt{d}$.

\begin{figure*}[t!]
\centering
\begin{tikzpicture}[scale=0.85]
\draw[step=1cm,gray,very thin] (-3,-2) grid (5,6);
\draw[thick,->] (-2,0) -- (4,0);
\draw[thick,->] (0,-1) -- (0,5);
\node at (-0.3,-0.3) {O};
\node at (4.3,-0.3) {$m$};
\node at (-0.3,5.3) {$d$};
\node at (1.7,-0.3) {$n^1$};
\node at (-0.3,2) {$n^1$};
\node at (-0.3,4) {$n^2$};
\draw[gray,thick,dashed] (-2,-0) -- (3,5);
\draw[gray,thick,dashed] (3,-1) -- (-2,4);
\draw[gray,thick,dashed] (2,-1) -- (2,5);
\draw[gray,thick,dashed] (-2,3) -- (4,3);
\fill[fill=blue] (0,2) -- (1,3) -- (2,3) -- (2,0);
\fill[fill=blue!30] (1,3) -- (2,4) -- (2,3);
\node at (3.5, 5.5) {$\frac{qm}{s}$};
\node at (3.3, -1.3) {$s$};
\node at (4.3, 3) {$\frac{qm}{\sigma^2s^2}$};
\end{tikzpicture}
\begin{tikzpicture}[scale=0.85]
\draw[step=1cm,gray,very thin] (-3,-2) grid (5,6);
\draw[thick,->] (-2,0) -- (4,0);
\draw[thick,->] (0,-1) -- (0,5);
\node at (-0.3,-0.3) {O};
\node at (4.3,-0.3) {$m$};
\node at (-0.3,5.3) {$d$};
\node at (1.7,-0.3) {$n^1$};
\node at (-0.3,2) {$n^1$};
\node at (-0.3,4) {$n^2$};
\draw[gray,thick,dashed] (-4/3,-1) -- (8/3,5);
\draw[gray,thick,dashed] (4,-1) -- (-2,2);
\draw[gray,thick,dashed] (2,-1) -- (2,5);
\draw[gray,thick,dashed] (-2,1) -- (4,4);
\fill[fill=blue] (0,1) -- (2,4) -- (2,0);
\fill[fill=blue!30] (1,2.5) -- (2,4) -- (2,3);
\node at (3, 5.5) {$\frac{qm}{s}$};
\node at (4.3, -1.3) {$s$};
\node at (4.5, 4.2) {$\frac{qm}{\sigma^2s^2}$};
\end{tikzpicture}
\caption{Phase diagram of perfect recovery for $t=1$ (left) and $t=2$ (right).
Here $\sigma^2 \asymp 1/\sqrt{n}$.} \label{fg:diag}
\end{figure*}

Our main finding is that when the parameters $n$, $d$, $m$, $\sigma$, and $q$
satisfy certain conditions,
the solution found by \eqref{eq:mp_align} recovers the true permutation $\pi^*$ with high probability.
\begin{theorem}[Perfect recovery] \label{th:perf}
Let the matching problem be defined in Section~\ref{sc:model},
and we solve it using \eqref{eq:mp_align}.
With probability approaching $1$ as $n \to \infty$
we recover both the true vertex features and the matching if
\begin{equation*}
\min\biggl\{s, \frac{qm}{s},
\frac{qm}{\sigma^2 s^2}\biggr\}
\gg \log n + \log d.
\end{equation*}
\end{theorem}

Intuitively, there is a bias--variance trade-off in aggregating feature 
information from neighbors in a geometric graph:
The noise cancels out by averaging, thus reducing the variance in estimation,
while borrowing features from connected vertices increases the bias.
Therefore, the parameters must satisfy certain conditions for the GNN
to achieve perfect recovery.
The proofs make use of the concentration of measure to derive various upper
and lower bounds on important quantities such as the support of feature
vectors and the neighborhood size.
Finer and more detailed results can be found in
Theorems~\ref{th:feat} and~\ref{th:match}
in the appendix along with their proofs.

\begin{remark}[Reparameterization]
The parameter regime is specified by a mixture of $s$ and $m$ for the simplicity
of the presentation.
However, as defined in \eqref{eq:s}, $s$ is determined by $m$ and vice versa.
Since we are more interested in graph properties,
we choose the more natural parameter $m$.
Replacing $s$ with the definition in \eqref{eq:s}, we immediately have
\begin{equation*}
\begin{split}
&\min\biggl\{\sqrt{d\biggl(\frac{m}{n}\biggr)^{1/t}},
\sqrt{\frac{q^2 m^2}{d}\biggl(\frac{n}{m}\biggr)^{1/t}},
\frac{qm}{\sigma^2 d}\biggl(\frac{n}{m}\biggr)^{1/t}\biggr\}\\
&\quad\gg \log n + \log d.
\end{split}
\end{equation*}
Since $q$ also affects the edge density of the graph but in a less sophisticated way
(only through subsampling),
for the interest of discussion,
we also fix it to be a constant.
Now we are ready to present an interplay between $n$, $d$, $m$, and $\sigma$
in several graph density regimes.
Notably for different graph densities $m$,
we identify the permissible range for $d$,
within which exact matching recovery is feasible,
contingent on the noise level $\sigma^2$ being adequately small.
We similarly derive the maximum size the noise $\sigma^2$ can take for which
perfect recovery is still possible. 
We list the results in Table~\ref{tb:combined_phase} when $t=1$ and $t=2$
for $m = a \log^{2+\epsilon} n, \epsilon > 0$ (sparse)
and
$m = b n^{\alpha}$ with $0 < \alpha < 1$ (intermediate)
and $\alpha = 1$ (dense),
where $a, b > 0$ are absolute constants. 
\end{remark}
\begin{table*}[h]
\centering
\caption{Recovery conditions for different thresholds $t$ and sparsity levels $m$.}
\label{tb:combined_phase}
\begin{tabular}{c|c|c}
\hline
m& $(\log n)^{2+\epsilon}$ & $n^\alpha$ \\ 
\hline
$d$ for $t = 1$ & $\displaystyle {(\log n)^{\epsilon}n}\gg d\gg \frac{n}{(\log n)^{\epsilon}}$ & $\displaystyle \frac{n^{1+\alpha}}{(\log n)^2}\gg d\gg (\log n)^2n^{1-\alpha}$ \\[2ex]
$\sigma^2$ for $t = 1$ & $\sigma^2\ll \frac{1}{(\log n)^{1-\epsilon}}$ & $\sigma^2\ll \frac{n^\alpha}{(\log n)^3}$ \\[2ex]
\hline
$d$ for $t = 2$ & $\displaystyle {\sqrt{n}}{{(\log n)^{1+\frac{3}{2}\epsilon}}}\gg d\gg \sqrt{n}(\log n)^{1-\frac{\epsilon}{2}}$ & $\displaystyle \frac{n^{\frac{1+3\alpha}{2}}}{(\log n)^2}\gg d\gg (\log n)^2(\sqrt{n})^{1-\alpha}$ \\[2ex]
$\sigma^2$ for $t = 2$ & $\sigma^2\ll \frac{1}{(\log n)^{1-\epsilon}}$ & $\sigma^2\ll \frac{n^\alpha}{(\log n)^3}$ \\[2ex]
\hline
\end{tabular}
\end{table*}
\begin{remark}[Phase diagram]
Note that only the last term in the bound depends on the noise parameter $\sigma$.
All previous terms are functions of $n$, $d$, and $m$.
If we consider only the intermediate regime when $m \asymp n^{\alpha}$ and $d \asymp n^{\beta}$,
we can visualize the theorem as a diagram in the space of $m$ and $d$ as in Figure~\ref{fg:diag}.
The entire colored region is specified by the first two terms (and $m \le n$) without $\sigma$.
The term that involves $\sigma$ ``cuts through'' the region and
moves down as $\sigma$ increases as a power of $n$.
The dark blue region is where perfect recovery is possible for the GNN.
\end{remark}

\begin{remark}[Correlated random graph matching]
Information-theoretic results on correlated random graph matching have been 
investigated in a line of research~\citep{cullina2016improved,wu2022settling,ding2023matching}
and the sharp threshold for exact recovery was established for the Erd\H{o}s--R\'enyi model.
These results were later extended to inhomogeneous random graphs including random geometric graphs by~\citet{racz2023matching}.
Their result suggests that using the incomplete graphs alone,
one can find the exact matching when $mq^2 \gg \log n$,
compared to $\min\{\frac{mq}{s}, \frac{mq}{\sigma^2 s^2}\} \gg \log n$
as required in the theorem.
Therefore, when $q \ll \min\{\frac{1}{s}, \frac{1}{\sigma^2 s^2}\}$,
i.e., the graph is very sparse, GNN is able to recover the matching
while the k-core estimator cannot.
One additional note here is that the algorithms used to prove the information-theoretic
results are usually not computationally efficient.
For instance, the MAP estimator in \citep{cullina2016improved} would require inspecting
all permutations.
\end{remark}

\begin{remark}[Trainability of the GNN]
We focus on understanding the message passing in this work,
and the GNN used for the alignment does not contain weight matrices.
Nevertheless, the threshold in the activation function $\eta$ can, in fact, be trained.
Since we assume the RIG parameters are known to us,
the threshold is chosen ``optimally'' in the algorithm.
However, when dealing with real-world data where the parameters may not be known
even if the generative model assumptions hold,
the threshold may be learned from the data.
We further investigate this point empirically in our real-world data experiments.
\end{remark}

We have the following negative result that establishes the conditions
under which the GNN cannot recover the matching perfectly.
\begin{theorem}[Impossibility of perfect recovery]
Assume that $\min\{s, \frac{qm}{s}\} \ge c$ for a constant $c >0$.
Then there exists a constant $\delta > 0$ such that $\PP(\hat{\pi} \ne \pi^*) \ge \delta'$
for all $\delta' \in [0, \delta]$ if
\begin{equation*}
\frac{qm}{\sigma^2s^2} \le C_\delta
\end{equation*}
for some constant $C_\delta > 0$.
\end{theorem}
To prove the impossibility result,
we carefully characterize the event in which two vertices are misaligned
with each other in the two observed graphs.
Hence perfect recovery is not possible when this event occurs.
The proof can be found in the appendix.

\begin{remark}
We note that the conditions $\min\{s, \frac{qm}{s}\} \ge c$ assumed here
are also necessary for perfect recovery.
These regularity conditions ensure that the vertices are unique
and that the signal is sufficiently strong in the graph.
The primary distinction between the possibility and impossibility bounds
lies in the term $\frac{qm}{\sigma^2s^2}$
which is also the only term that involves the noise parameter $\sigma$.
For perfect recovery, it must be $\gg \log n + \log d$,
whereas for impossibility, it must be bounded.
This implies that the bound concerning the noise parameter $\sigma$ is tight up to logarithmic factors.
Notably combined with the perfect recovery results,
in the dense regime $m \asymp n^\alpha$,
we show that the algorithm can tolerate noise level $\sigma$
that grows as a function of $n$.
\end{remark}

Instead of using the graph neural network,
one can obtain a matching by directly solving an alignment problem
from the noisy vertex features $\Y$ and $\Y'$:
\begin{equation} \label{lin_algo}
\begin{split}
\tilde{\pi} &= \argmin_{\pi \in \cP(n)} \tilde{\ell}(\pi)
\coloneqq \sum_{i=1}^n \norm{\y_i - \y'_{\pi(i)}}^2\\
&= \argmin_{\pi \in \cP(n)}
\sum_{i=1}^n \inner{\y_i}{\y'_{\pi(i)}}.
\end{split}
\end{equation}
We call this the linear method.
It is worth mentioning that this method has been widely used for alignment problems
and also attracted significant theoretical analysis under different settings.
However, as we will show in the next theorem,
when the noise is large or the dimension is high,
directly solving the assignment problem will not achieve perfect recovery
with at least constant probability.
\begin{theorem}[Impossibility of perfect recovery with vertex features]
If $\sigma^2\ge2{s}\bigl({1+\frac{K}{n}}\bigr)^{-2}$ for any $K>4$ we have that
\begin{equation*}
\PP(\tilde{\pi} = \pi^*) \le e^{-\frac{K-4}{4}}.
\end{equation*}
Hence if $\sigma^2 \gg s$, the probability of perfect recovery converges to $0$. 
If instead $\frac{1}{4}\sigma^2\le s\le \frac{d}{2}$ and 
$\sigma^2\gg\frac{s}{\sqrt{d \log n}}$ then
\begin{equation*}
\lim_{n\to\infty} \PP(\tilde{\pi} = \pi^*) = 0.
\end{equation*}
\end{theorem}
To prove the first part of the theorem when the noise parameter is large,
we make use of several information-theoretic inequalities.
The proof of the second part involves analyzing the misalignment event.
Finer and more detailed results are postponed to
Propositions~\ref{pp:low1} and~\ref{pp:low2},
along with their proofs in the appendix.

\begin{remark}
The impossibility result is split into two regimes:
the signal-to-noise ratio being small and large.
We are more interested in how these bounds compare to those achieved
by the graph neural network.
To make a direct comparison,
we may translate the parameter $s$ to $m$, 
although it does not have a clear physical meaning for aligning features.
In one scenario, the message is clear:
When $\sigma$ is at least a constant,
in all parameter regimes where the GNN has perfect recovery,
the linear method would fail.
\end{remark}

\section{EXPERIMENTS}
We validate our findings through experiments on artificially generated graphs
as specified by our definitions,
as well as two real-world network datasets.
With a slight deviation from the theoretical results,
we present the matching accuracy in the experimental results
for computational concerns and practical considerations.
This corresponds to the notion of partial recovery,
the theory of which we leave for future work.
All experiments are averaged over $5$ independent runs,
and we report the two standard deviations
plotted as the shaded areas along the curve.
(Note that due to the small variance in a few experiments,
some shaded areas are very narrow and may not be clearly visible.)
The results were produced on a shared research computing cluster.
Required resources are $4$ computing units with $20$G RAM.
Experiments for each set of plots typically finish within $20$ minutes.
Code for generating the experimental results is available at \url{https://github.com/6457/matchgnn/}.

\subsection{Synthetic graphs}
We first generate graphs according to the model introduced in Section~\ref{sc:model}.
In Figure~\ref{fg:err},
we show the alignment accuracy, which is defined as the proportion of correctly matched pairs.
We first fix the sparsification parameter $q$ and vary the variance parameter
$\sigma$, and then fix $\sigma$ and vary $q$.
As suggested by the theory,
the linear method fails to recover the alignment when the feature noise is large,
while the GNN is less prone to being affected by noise.
The accuracy of the GNN also improves as $q$ grows.
Note that for the parameters chosen,
the original graph is already very sparse.
The average degree is about $56$ compared to $n = 4000$ nodes.
This explains why the error of the GNN is high when $q$ is small.
\begin{figure}[ht]
\centering
\subfigure[Impact of the noise parameter $\sigma$.]{\includegraphics[width=0.415\textwidth]{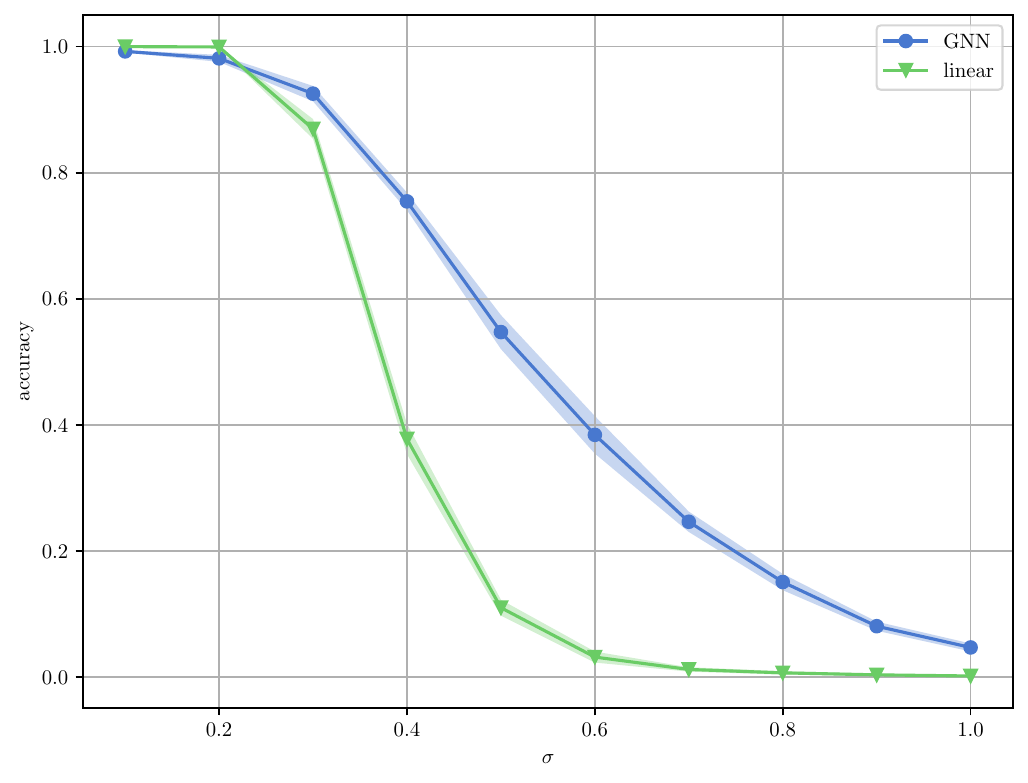}}
\subfigure[Impact of the sparsification parameter $q$.]{\includegraphics[width=0.415\textwidth]{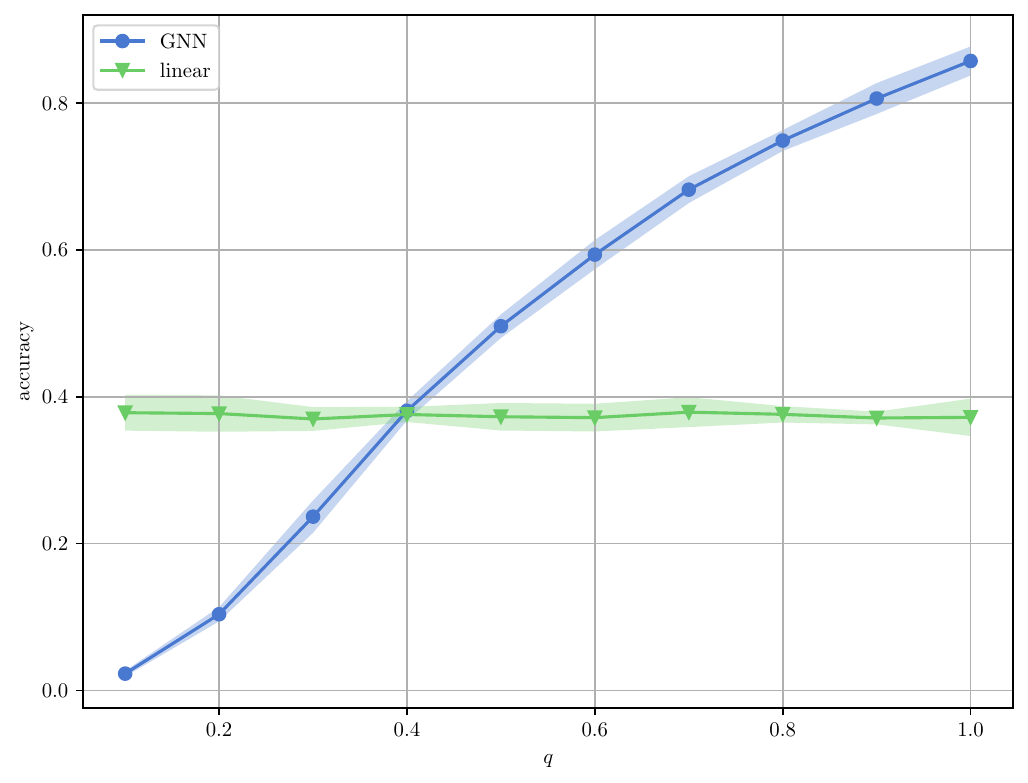}}
\caption{Comparison of the GNN and the linear method.
Parameters in the experiments are $n=4000$, $d=200$, $s=10$, $t=3$.
For (a), $q=0.8$ is fixed
and $\sigma$ ranges from $0.1$ to $1$ in $0.1$ increments.
For (b),
$\sigma=0.4$ and $q$ changes from $0.1$ to $1$ in $0.1$ increments.}
\label{fg:err}
\end{figure}

\begin{figure*}[t!]
\centering
\subfigure[Cora]{
\includegraphics[width=0.235\textwidth]{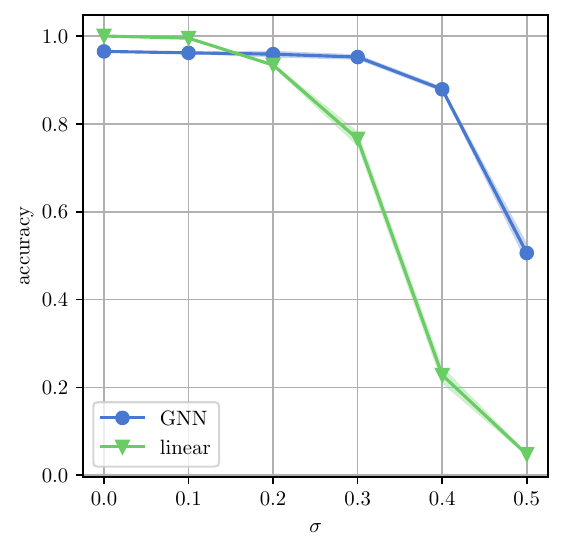}
\includegraphics[width=0.235\textwidth]{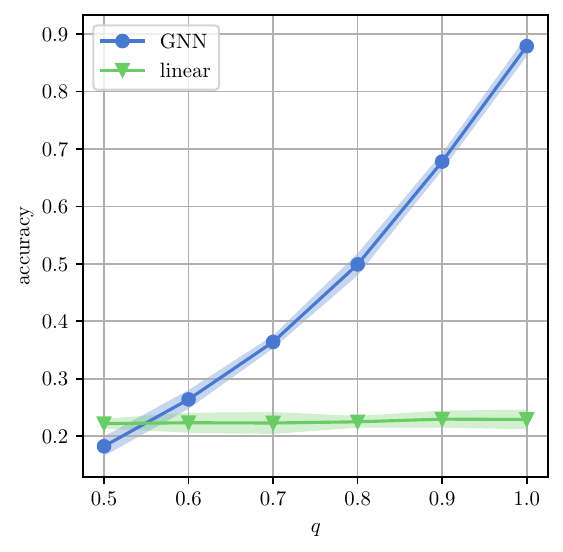}
}
\subfigure[CiteSeer]{
\includegraphics[width=0.235\textwidth]{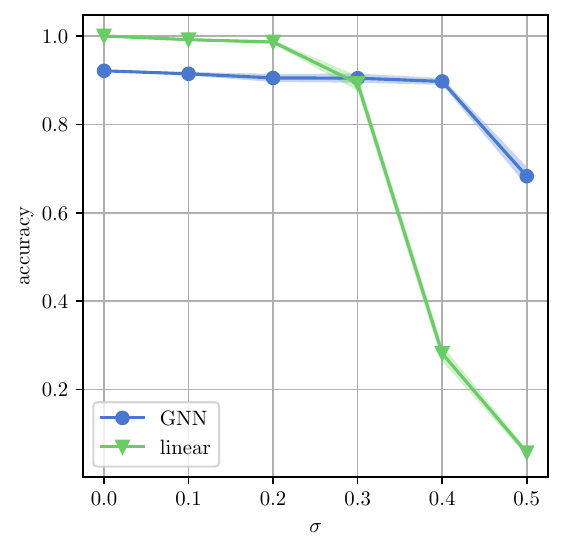}
\includegraphics[width=0.235\textwidth]{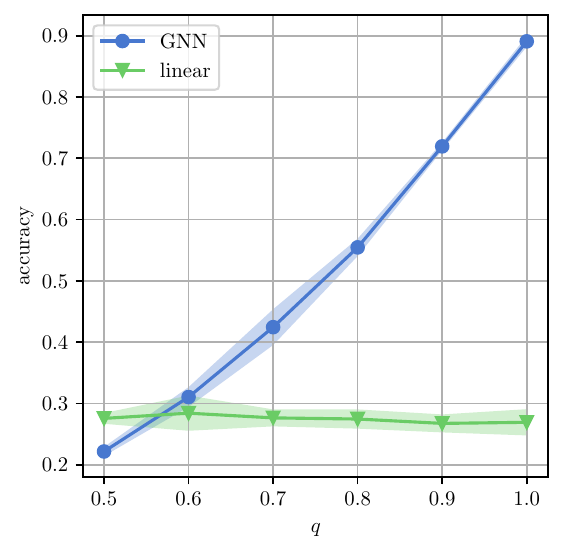}}
\caption{Comparison of the GNN and the linear method on real-world datasets.
In the plots on the left of each group, $q = 1$ is fixed (using all edges from the datasets)
and $\sigma$ varies from $0$ to $0.5$ in $0.1$ increments.
In the plots on the right of each group,
$\sigma = 0.4$ is fixed and $q$ varies from $0.5$ to $1$ in $0.1$ increments.
}
\label{fg:acc_real}
\end{figure*}

\subsection{Real-world datasets}
We perform the alignment task on two public benchmark datasets, Cora and CiteSeer~\citep{sen2008collective},
which have been widely used to evaluate the performance of GNNs~\citep{yang2016revisiting}.
Both datasets contain research papers with their bag-of-words representation.
The dataset statistics may be found in Table~\ref{tb:stats}.
\begin{table}[h!]
\centering
\caption{Summary of datasets.}
\label{tb:stats}
\begin{tabular}[t]{c|c|c|c}
\hline
Dataset & \# Vertices & \# Edges & \# Features\\ 
\hline
Cora & $2,708$ & $5,429$ & $1,433$\\
CiteSeer & $3,327$ & $4,732$ & $3,703$\\
\hline
\end{tabular}
\end{table}

The evaluation task follows a similar approach to \cite{yan2021bright}.
We treat the original network as our ``ground truth'' graph,
and then create two copies of the graph by following the generating process
of NIRIG as in Definition~\ref{df:nirig}:
Each word feature is perturbed with Gaussian noise
and each edge is sampled with a fixed probability.
The goal is to recover the article correspondence between them.
The experiments are similar to those with the synthetic data.
One additional tuning parameter here is the threshold in the activation function,
as it is not known to us,
even if the real-world networks are indeed geometric graphs.
(We delve into this in the next set of experiments.)
We found that the accuracy is not very sensitive to it in a large range
within which the threshold is small.
So we fix the threshold $b = 0.1$ in this part.
We also replace the hard thresholding with a soft one for practical purposes.
The results are presented in Figure~\ref{fg:acc_real}.

Next, we evaluate the recovery accuracy of the GNN with different thresholds
on our two real-world datasets (Figure~\ref{fg:acc_b}).
\begin{figure}[h!]
\centering
\subfigure[Cora]{
\includegraphics[width=0.42\textwidth]{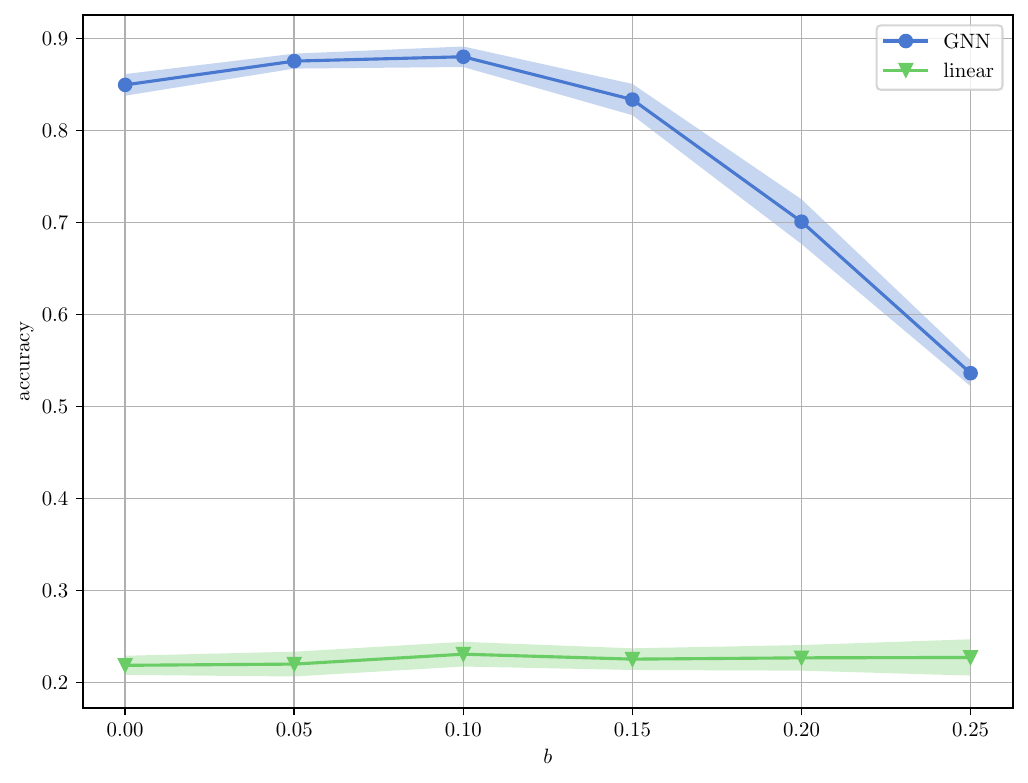}
}
\subfigure[CiteSeer]{
\includegraphics[width=0.42\textwidth]{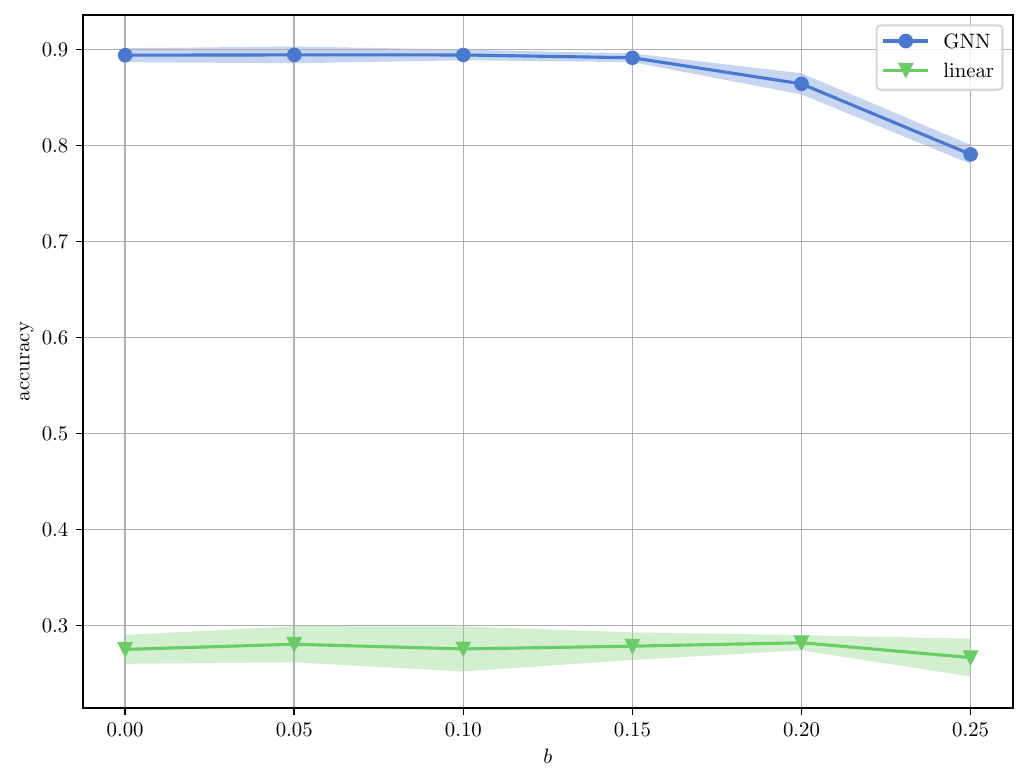}
}
\caption{Impact of the threshold parameter on real-world datasets.
We fix $q=1$ and $\sigma=0.4$.}
\label{fg:acc_b}
\end{figure}
Interestingly, the two datasets exhibit very different response curves
to the threshold.
In Cora, the accuracy first grows when the threshold becomes larger
and then drops.
The U-shape curve suggests that there may exist an optimal choice of the threshold.
Meanwhile, in CiteSeer,
the accuracy remains the same in a large range of thresholds
and only starts to decrease much later.
This may be due to the fact that Cora consists of only machine learning 
publications,
hence the number of overlapping words is comparably larger than
in CiteSeer, which has more diverse fields in computer science,
resulting in more accurate words from the neighborhood.

\section{OPEN PROBLEMS}
In this paper,
we presented various possibility and impossibility results 
for perfect recovery under the noisy sparse feature setting.
Similar inquiries could be made regarding other feature distributions,
such as Gaussian features, which are the central object of study
by~\citet{kunisky2022strong},
or spherical distributions arising from high-dimensional random geometric 
graphs~\citep{devroye2011high,bubeck2016testing,brennan2020phase,liu2023phase}.
However, it would be less clear what role message passing plays
under these settings.
Nevertheless, averaging over the neighbors should still be
effective thanks to the symmetry of the distributions.

Perfect recovery is the primary objective of this work.
It would also be interesting and potentially useful in practice
to explore partial recovery for both vertex features and the unknown alignment.
There has been considerable research on partial recovery
in either aligning Gaussian features~\citep{kunisky2022strong}
or random graph matching~\citep{cullina2019partial}.
However, these questions still remain open for the current model.
The techniques developed in the proofs of our perfect recovery results
should bring us closer to finding the answers.

\subsubsection*{Acknowledgements}
We thank Yuling Yan for many insightful discussions and helpful suggestions on an early draft of the paper,
as well as the reviewers for their useful comments.

\clearpage

\thispagestyle{empty}
\onecolumn

\appendix
\section{PRELIMINARIES}
In this section we present some preliminary results on the tails of different random variables. 
These will be used later in our proofs.
\begin{proposition}[Chernoff bound] \label{pp:chernoff}
Let $X_1, \ldots, X_n$ be independent Bernoulli random variables
and $N = \sum_{i=1}^n X_i$ be their sum.
Denote $\EE[N] = \mu$ the expected value of the sum.
Then for any $\delta \ge 0$,
\begin{equation*}
\PP(N \ge (1+\delta)\mu) \le \exp\biggl(-\frac{\delta^2 \mu}{2+\delta}\biggr)
\end{equation*}
and for any $0 < \delta < 1$,
\begin{equation*}
\PP(N \le (1-\delta)\mu) \le \exp\biggl(-\frac{\delta^2 \mu}{2}\biggr).
\end{equation*}
\end{proposition}

The following lower bound for the tail of a nonnegative random variable
is due to~\cite{paley1932some}
(see also \citep[Inequality II, p.~8]{kahane1985some}).
\begin{proposition}[Paley--Zygmund] \label{pp:PZ}
Let $Z$ be a nonnegative random variable with finite variance.
For $0 \le \theta \le 1$,
\begin{equation*}
\PP(Z \ge \theta \EE[Z]) \ge (1 - \theta)^2 \frac{\EE[Z]^2}{\EE[Z^2]}.
\end{equation*}
\end{proposition}
We exploit this proposition to deduce the following tail lower bound
for the dot product of Gaussian vectors.
\begin{proposition}[Lower bound for Gaussian inner product]
\label{pr:inner_lower}
Let $\x, \y \sim \cN(\0, \I_d)$ be two independent standard Gaussian random 
vectors.
We have that for any $u \le \sqrt{d}/16$, the following holds
\begin{equation*}
\PP(\inner{\x}{\y} \ge u\sqrt{d})
\ge (1 - e^{-4u^2})^2 e^{-22u^2}.
\end{equation*}
\end{proposition}

\begin{proof}
Denote $Z = \exp(t\inner{\x}{\y})$ for $0 < t < 1/2$.
By the proof of \cite[Proposition~2.14]{liu2023probabilistic}, we have
\begin{equation*}
\EE[Z] = \EE[e^{t\inner{\g_1}{\g_2}}] = \prod_{i=1}^{d} \EE[e^{t \g_{1,i} \g_{2,i}}]
= (1 - t^2)^{-d/2}.
\end{equation*}
Hence this implies that for every $t<\frac{1}{4}$ the following holds
\begin{equation*}
\EE[Z^2] = \EE[e^{2t\inner{\g_1}{\g_2}}] = (1 - 4t^2)^{-d/2}.
\end{equation*}
By Paley--Zygmund (Proposition~\ref{pp:PZ}),  for all $\theta\in (0,1)$
we have
\begin{equation*}
\PP\biggl(\inner{\x}{\y} \ge -\frac{d}{2t} \log(1 - t^2)
+ \frac{1}{t} \log \theta\biggr)
\ge (1 - \theta)^2 \biggl(\frac{1 - t^2}{\sqrt{1 - 4t^2}}\biggr)^{-d}.
\end{equation*}
By taking $t = 4u/\sqrt{d}$ and $\theta = \exp(-4u^2)$ for $u \le \sqrt{d}/16$,
we have that
\begin{equation*}
-\frac{d}{2t} \log(1 - t^2) + \frac{1}{t} \log \theta
= -\frac{d^{3/2}}{8u} \log\biggl(1 - \frac{16u^2}{d}\biggr) - u \sqrt{d}
\ge u\sqrt{d}
\end{equation*}
where we used $\log (1 - x) \le -x$, and 
\begin{equation*}
\biggl(\frac{1 - t^2}{\sqrt{1 - 4t^2}}\biggr)^{-d}
= \biggl(1 + 2t^2\biggl(\frac{1 + t^2/2}{1 - 4t^2}\biggr)\biggr)^{-d/2}
\ge \biggl(1 + \frac{44u^2}{d}\biggr)^{-d/2}
\ge e^{-22u^2}
\end{equation*}
where we used $(1+x/d)^d \le e^x$.
\end{proof}

\begin{proposition}[Lower bound for lazy random walk tail] \label{pp:lazy_tail}
Let $S_n = \sum_{i=1}^n X_i$ be a lazy random walk such that
$\PP(X_i = 1) = \PP(X_i = -1) = p$ and $\PP(X_i = 0) = 1 - 2p$.
Assume $pn \ge 1$.
Then for $0 \le u \le pn/\beta$, there exist constants $c_{\beta}, C_\beta > 0$
that depend only on $\beta$ such that
\begin{equation*}
\PP(S_n \ge u) \ge c_{\beta}\exp\biggl(- C_\beta\frac{u^2}{pn}\biggr).
\end{equation*}
\end{proposition}

\begin{proof}
Let $Y_i \in \{-1, 1\}$ for $i \in [n]$ be i.i.d.\ Rademacher random variables, i.e.,
$\PP(Y_i = 1) = \PP(Y_i = -1) = 1/2$.
Let $Z_i \sim \Bern(2p), i \in [n]$ be i.i.d.\ Bernoulli random variables.
Then, we can write $X_i = Y_i Z_i$.
Let $N_n = \sum_{i=1}^n Z_i$ be the sum of $Z_i$'s.
Then given $N_n$, $S_n$ is a sum of $N_n$ i.i.d.\ Rademacher random variables.
By \citep[Corollary~4]{zhang2020non},
for any $\beta > 1$, there exists constants $c_\beta, C_\beta > 0$,
such that for all $0 \le u \le N_n/\beta$,
\begin{equation*}
\PP(S_n \ge u \mid N_n) \ge c_\beta \exp\biggl(-C_\beta \frac{u^2}{N_n}\biggr).
\end{equation*}
Chernoff bound (Proposition~\ref{pp:chernoff}) gives
\begin{equation*}
\PP\biggl(N_n \le \frac{pn}{2}\biggr) \le \exp\biggl(-\frac{pn}{8}\biggr).
\end{equation*}
Therefore, we conclude that
\begin{equation*}
\begin{split}
\PP(S_n \ge u) &\ge \PP\biggl(S_n \ge u, N_n \ge \frac{pn}{2}\biggr)
= \PP\biggl(S_n \ge u \biggm\vert N_n \ge \frac{pn}{2}\biggr)
\biggl(1 - \PP\biggl(N_n \le \frac{pn}{2}\biggr)\biggr)\\
&\ge c_\beta \exp\biggl(-C_\beta \frac{8u^2}{pn}\biggr)
\biggl(1 - \exp\biggl(-\frac{pn}{8}\biggr)\biggr).
\end{split}
\end{equation*}
The claim directly follows.
\end{proof}

\section{AUXILIARY LEMMAS}
A necessary condition for perfect recovery is that
the feature vector for each vertex is unique.
The following proposition tells us that this happens with high probability
in RIG when $s$ is sufficiently large.
\begin{proposition}[No two vectors are the same]\label{pp:unique}
Let $\x_i$'s be defined in Definition~\ref{df:rig}.
If $s \ge (1 + c)\log n$ for some $c > 0$,
then no two vectors in $\{\x_1, \ldots, \x_n\}$ are the same
with probability at least $1 - n^{-2c}$.
\end{proposition}

\begin{proof}
As the entries $\x_i(k),\x_j(k)\overset{i.i.d}{\sim }\Bern(\frac{s}{d})$,
we obtain that 
\begin{equation*}
\PP(\x_i(k) \ne \x_j(k)) = \frac{2s}{d}.
\end{equation*}
Hence, by independence of the entries, we have
\begin{equation*}
\PP(\x_i = \x_j)=\PP(\x_i(k) = \x_j(k)~\forall k\le d) = \biggl(1 - \frac{2s}{d}\biggr)^d.
\end{equation*}
Hence by a union bound argument we obtain that 
\begin{equation*}
\PP(\exists i\ne j\le n~\mathrm{s.t.}~\x_i=\x_j)\le \sum_{i<j} P(\x_i=\x_j)
\le \binom{n}{2}\biggl(1 - \frac{2s}{d}\biggr)^d
\le e^{- 2(s - \log n)}.
\end{equation*}
The claim directly follows.
\end{proof}

Let $S_i$ be the support of $\x_i$.
Since $\x_i(k), k \in d$ are i.i.d.\ Bernoulli random variables with parameter $s/d$,
applying Chernoff bound (Proposition~\ref{pp:chernoff}) to
$\abs{S_i} = \sum_{k=1}^d \x_i(k)$ directly gives us
the following lower and upper deviation bounds for $\abs{S_i}$.
\begin{lemma} \label{lm:S_lower}
For all $i \in [n]$ and any $\delta \ge 0$,
\begin{equation*}
\PP(\abs{S_i} \ge (1 + \delta)s) \le \exp\biggl(-\frac{\delta^2 s}{2+\delta}\biggr).
\end{equation*}
\end{lemma}

\begin{lemma} \label{lm:S_upper}
For all $i \in [n]$ and $0 < \delta < 1$,
\begin{equation*}
\PP(\abs{S_i} \le (1 - \delta)s) \le \exp\biggl(-\frac{\delta^2 s}{2}\biggr).
\end{equation*}
\end{lemma}
Define the event
\begin{equation}\label{eq:Si_s}
\cE_i \coloneqq \biggl\{\frac{s}{2} \le \abs{S_i} \le 2s\biggr\}.
\end{equation}
Lemma~\ref{lm:S_upper} and~\ref{lm:S_lower} immediately suggest that
\begin{equation} \label{eq:E_prob}
\PP(\cE_i^c) \le \exp\biggl(-\frac{s}{8}\biggr) + \exp\biggl(-\frac{s}{3}\biggr)
\le 2\exp\biggl(-\frac{s}{8}\biggr).
\end{equation}

In the proofs, we use a lot of conditional arguments.
Most frequently, $s/2 \le \abs{S_i} \le 2s$,
which happens with high probability from previous lemmas,
is usually assumed.
The following elementary inequality becomes handy in such situations.
\begin{proposition} \label{pp:cond}
For any two events $A$ and $B$,
\begin{equation*}
\PP(A) \le \PP(A \mid B) + \PP(B^c).
\end{equation*}
\end{proposition}

\begin{proof}
By law of total probability,
\begin{equation*}
\PP(A) = \PP(A \cap B) + \PP(A \cap B^c)
= \PP(A \mid B) \PP(B) + \PP(A \mid B^c) \PP(B^c)
\le \PP(A \mid B) + \PP(B^c),
\end{equation*}
where we used the fact that the probability measure is less than or equal to $1$.
\end{proof}

For all vertex $i \in [n]$, let
\begin{equation} \label{eq:Fi}
\cF_i \coloneqq \biggl\{\frac{1}{2^{t+2}}\cdot mq\le
\abs{N_i} \le 2^{t+2}e^t \cdot mq\biggr\}.
\end{equation}

\begin{lemma} \label{lm:N_i}
The following holds for all $i \in [n]$ when $n \ge (12t)^t m$,
\begin{equation*}
\PP(\cF_i^c) \le 2\exp\biggl(-\frac{s}{8}\biggr)
+ 2\exp\biggl(-\frac{mq}{2^{t+6}}\biggr).
\end{equation*}
\end{lemma}

\begin{proof}
By definition the neighborhood size is given by
$|N_i|=\sum_{j=1}^n \bbmone\{j \in N_i\}$.
Moreover we remark that conditionally on $\x_i$,
the edges $\big(\bbmone\{j \in N_i\}\big)_{j \in [n]}$
are independent and identically distributed.
We denote by $p_i \coloneqq \PP(j \in N_i \mid \x_i)$ the edge probability.
Hence we remark that conditionally on $\x_i$,
$\abs{N_i}$ is a sum of i.i.d. Bernoulli random variables with parameter $p_i$.

We first obtain upper and lower bounds for $p_i$  conditionally the event
$\cE_i$ defined in \eqref{eq:Si_s} holding.
For the rest of the proof,
all probabilities are conditioned on $\cE_i$ and we omit it from the condition
for simplicity of presentation.
We note that $\inner{\x_i}{\x_j} = \sum_{k=1}^d\x_i(k)\x_j(k)$
is a $\Binom(|S_i|,\frac{s}{d})$ random variable.
Recall that the edges are a subsample of the original graph with probability $q$.
Hence the conditional probability that $j \in N_i$ is given by
\begin{equation*}
\begin{split}
\PP(j \in N_i \mid \x_i)
&= \PP(j \in N_i \mid \inner{\x_i}{\x_j} \ge t)
\PP(\inner{\x_i}{\x_j} \ge t \mid \x_i)
= q\PP(\inner{\x_i}{\x_j} \ge t \mid \x_i)\\
&= q \sum_{k=t}^{\abs{S_i}} \binom{\abs{S_i}}{k}\biggl(\frac{s}{d}\biggr)^{k}
\biggl(1 - \frac{s}{d}\biggr)^{\abs{S_i} - k} \eqqcolon p_i.
\end{split}
\end{equation*}
Then we can lower bound $p_i$ by
\begin{equation*}
\begin{split}
p_i
&= q \sum_{k=t}^{\abs{S_i}} \binom{\abs{S_i}}{k}\biggl(\frac{s}{d}\biggr)^{k}
\biggl(1 - \frac{s}{d}\biggr)^{\abs{S_i} - k}
\ge q \binom{\abs{S_i}}{t}\biggl(\frac{s}{d}\biggr)^{t}
\biggl(1 - \frac{s}{d}\biggr)^{\abs{S_i} - t}\\
&\overset{(a)}{\ge} q \biggl(\frac{\abs{S_i} s}{t d}\biggr)^t
\biggl(1 - \frac{\abs{S_i}s}{d}\biggr),
\end{split}
\end{equation*}
where we used the elementary inequalities $\binom{n}{k} \ge (n/k)^k$
and $(1 - x)^r \ge 1-rx$ for $x \le 1$ in $(a)$.
Hence, when $s/2 \le \abs{S_i} \le 2s$ and $d \ge 4s^2$ (or equivalently $n \ge (4t)^t m$),
we have
\begin{equation} \label{eq:pi_low}
p_i \ge \frac{q}{2^t} \biggl(\frac{s^2}{td}\biggr)^t
\biggl(1 - \frac{2s^2}{d}\biggr)
\overset{(a)}{\ge} \frac{1}{2^{t+1}}\cdot\frac{mq}{n}
\end{equation}
where to get $(a)$ we used \eqref{eq:s}.
We also have that for $s/2 \le \abs{S_i} \le 2s$ and $d \ge 12s^2/t$
(or equivalently $n \ge (12)^t m$ according to \eqref{eq:s}),
\begin{equation} \label{eq:pi_up}
\begin{split}
p_i &= q \sum_{k=t}^{\abs{S_i}} \binom{\abs{S_i}}{k}\biggl(\frac{s}{d}\biggr)^{k}
\biggl(1 - \frac{s}{d}\biggr)^{\abs{S_i} - k}
\overset{(a)}{\le} q \sum_{k=t}^{\abs{S_i}} \biggl(\frac{es\abs{S_i}}{kd}\biggr)^{k}
\le q \sum_{k=t}^{\abs{S_i}} \biggl(\frac{2es^2}{td}\biggr)^{k}\\
&\le q \sum_{k=t}^{\infty} \biggl(\frac{2es^2}{td}\biggr)^{k}
\le 2q \biggl(\frac{2es^2}{td}\biggr)^{t}
\overset{(b)}{=} 2^{t+1}e^t \cdot \frac{mq}{n},
\end{split}
\end{equation}
where we used $\binom{n}{k} \le (\frac{en}{k})^k$ in $(a)$
and the equality $(b)$ is by \eqref{eq:s}.
Together with Proposition~\ref{pp:chernoff}, we conclude that
\begin{equation*}
\begin{split}
\PP(\abs{N_i} \ge 2^{t+2}e^t \cdot mq)
&= \PP\biggl(\abs{N_i} \ge 2^{t+2}e^t \cdot \frac{mq}{(n-1)p_i} \cdot (n-1)p_i\biggr)\\
&\overset{(a)}{\le} \PP\biggl(\abs{N_i} \ge \biggl(1 + \frac{n+1}{n-1}\biggr) (n-1)p_i\biggr)
\le \exp\biggl(- \frac{(n+1)^2p_i}{2n}\biggr)\\
&\overset{(b)}{\le} \exp\biggl(- \biggl(1+\frac{1}{n}\biggr)^2\frac{mq}{2^{t+2}}\biggr)
\le \exp\biggl(- \frac{mq}{2^{t+2}}\biggr)
\end{split}
\end{equation*}
where we used \eqref{eq:pi_up} in $(a)$ and \eqref{eq:pi_low} in $(b)$,
and
\begin{equation*}
\begin{split}
\PP\biggl(\abs{N_i} \le \frac{1}{2^{t+2}} \cdot mq\biggr)
&= \PP\biggl(\abs{N_i} \le \frac{1}{2^{t+2}} \cdot \frac{mq}{(n-1)p_i} \cdot (n-1)p_i\biggr)\\
&\overset{(c)}{\le} \PP\biggl(\abs{N_i} \ge \biggl(1 - \frac{n-2}{2n-2}\biggr) (n-1)p_i\biggr)
\le \exp\biggl(-\frac{(n-2)^2 p_i}{8(n-1)}\biggr)\\
&\overset{(d)}{\le} \exp\biggl(-\frac{(n-2)^2 mq}{2^{t+4}n(n-1)}\biggr)
\le \exp\biggl(-\frac{mq}{2^{t+6}}\biggr)
\end{split}
\end{equation*}
where we used \eqref{eq:pi_low} in $(c)$ and \eqref{eq:pi_up} in $(d)$.

The lemma directly follows from Proposition~\ref{pp:cond} combined with \eqref{eq:Si_s}.
\end{proof}

\section{PROOFS OF THE PERFECT RECOVERY RESULTS}
Since $\hat{\pi}$ is the minimizer of the empirical risk $\ell$,
we have $\ell(\hat{\pi}) \le \ell(\pi^*)$.
Without loss of generality,
we assume that $\pi^*$ is the identity transformation in the proofs
hence $\pi^*(i) = i$ for all $i \in [n]$.
By triangle inequality, we have
\begin{equation}\label{ootw}
\ell(\hat{\pi}) \le \ell(\pi^*) = \sum_{i=1}^n \norm{\z_i - \z_i'}^2
\le \sum_{i=1}^n (\norm{\z_i - \x_i} + \norm{\z_i' - \x_i})^2
\le  2\sum_{i=1}^n \norm{\z_i - \x_i}^2 + 2\sum_{i=1}^n \norm{\z_i' - \x_i}^2.
\end{equation} 
We will show that with high probability both
$\sum_{i=1}^n \norm{\z_i' - \x_i}^2$ and
$\sum_{i=1}^n \norm{\z_i - \x_i}^2$ are equal to zero.
This will directly imply that with high probability
the empirical risk $\ell(\hat{\pi})$ is also zero.
\begin{theorem}[Perfect recovery for vertex features] \label{th:feat}
Let $G = (\Y, E) \sim \NIRIG(\X, \sigma, q)$ be defined as in Section~\ref{sc:model}
and let $\Z = (\z_1, \ldots, \z_d)$ be the output of the graph neural network described
in Section~\ref{sc:main}.
Then for $n \ge (12t)^t m$,
\begin{equation*}
\PP(\exists i\ \textrm{s.t.}\ \x_i \ne \z_i)
\le nd\biggl(3\exp\biggl(-\frac{mq}{5\cdot2^{t+7}s}\biggr)
+ 3\exp\biggl(-\frac{s}{175}\biggr)
+ \exp\biggl(-\frac{mq}{2^{t+7}s^2\sigma^2}\biggr)\biggr). 
\end{equation*}
\end{theorem}
By Proposition~\ref{pp:unique},
with probability at least $1 - n^2e^{- 2s}$
all vertex features $\X = (\x_1, \ldots, \x_n)$ are unique.
Therefore, $\pi^*$ is the unique minimizer of $\eqref{eq:mp_align}$.
Since by Theorem~\ref{th:perf} we know that $\ell(\hat{\pi}) = 0$ with high probability,
we immediately have the following theorem by Proposition~\ref{pp:cond}.
\begin{theorem}[Prefect recovery for vertex matching] \label{th:match}
Let $\hat{\pi}$ be the solution of \eqref{eq:mp_align}.
Then for $n \ge (12t)^t m$,
\begin{equation*}
\PP(\hat{\pi} \ne \pi^*)
\le n^2 \exp(-2s) + 2nd\biggl(3\exp\biggl(-\frac{mq}{5\cdot2^{t+7}s}\biggr)
+ 3\exp\biggl(-\frac{s}{175}\biggr)
+ \exp\biggl(-\frac{mq}{2^{t+7}s^2\sigma^2}\biggr)\biggr).
\end{equation*}
\end{theorem}

The rest of this section is devoted to proving Theorem~\ref{th:feat}.
The proof is centered around events defined as follows.
For all $i \in [n]$ and $k \in [d]$, let
\begin{equation} \label{eq:Aik}
\cA_{i,k} \coloneqq \biggl\{\frac{s}{t\abs{N_i}} \sum_{j \in N_i} \x_j(k)
+ \frac{s}{t\abs{N_i}}\sum_{j \in N_i} \beps_j(k)
\le \frac{1}{2} \biggr\}.
\end{equation}

The next four lemmas concern events that will directly lead to bounds on $\cA_{i,k}$.
\begin{lemma} \label{lm:eik}
For all $i \in [n]$,
\begin{equation*}
\PP\biggl(\frac{s}{t\abs{N_i}}\sum_{j \in N_i} \beps_j(k) \ge \frac{1}{4}
\biggm\vert \abs{N_i}\biggr)
\le \exp\biggl(-\frac{t^2\abs{N_i}}{32s^2\sigma^2}\biggr).
\end{equation*}
\end{lemma}

\begin{proof}
By definition of $\beps_j$, $\beps_j(k)$ are i.i.d.\ $\cN(0, \sigma^2)$ for $j \in N_i$.
Hence,
\begin{equation*}
\PP\biggl(\frac{s}{t\abs{N_i}}\sum_{j \in N_i} \beps_j(k) \ge \frac{1}{4}
\biggm\vert \abs{N_i}\biggr)
= \Phi^c\biggl(\frac{t\sqrt{\abs{N_i}}}{4s\sigma}\biggr).
\end{equation*}
By the upper tail bound of standard normal distribution
(see, e.g., \citep[(2.7)]{wainwright2019high}),
\begin{equation*}
\Phi^c\biggl(\frac{t\sqrt{\abs{N_i}}}{4s\sigma}\biggr)
\le \exp\biggl(-\frac{t^2\abs{N_i}}{32s^2\sigma^2}\biggr).
\end{equation*}
The claim directly follows.
\end{proof}

\begin{lemma} \label{lm:xik_0}
For all $i \in [n]$, when $n \ge 8^t m$,
\begin{equation*}
\PP\biggl(\frac{s}{t\abs{N_i}} \sum_{j \in N_i} \x_j(k) \ge \frac{1}{4}
\biggm\vert \x_i(k) = 0, \abs{N_i}\biggr)
\le \exp\biggl(-\frac{t\abs{N_i}}{24s}\biggr).
\end{equation*}
\end{lemma}

\begin{proof}
First we remark that knowing $\x_i(k)=0$,
the coordinates $(\x_j(k))_{j \in N_i}$ are i.i.d.\ $\Bern(s/d)$ random variables.
Hence $\EE[\sum_{j \in N_i} \x_j(k)\mid \abs{N_i}] = s\abs{N_i}/d$.
Therefore by taking $\delta = \frac{td}{4s^2}$ in Proposition~\ref{pp:chernoff}, we have
\begin{equation*}
\PP\biggl(\frac{s}{t\abs{N_i}} \sum_{j \in N_i} \x_j(k) \ge \frac{1}{4}
\biggm\vert \x_i(k) = 0, \abs{N_i}\biggr)
\le \exp\biggl(-\biggl(\frac{td}{4s^2}-1\biggr)^2\cdot
\biggl(\frac{td}{4s^2}+1\biggr)^{-1}\cdot\frac{s\abs{N_i}}{d}\biggr).
\end{equation*}
For $n \ge 8^t m$,  we have $\frac{td}{4s^2} \ge 2$.
Hence,
\begin{equation*}
\biggl(\frac{td}{4s^2}-1\biggr)^2\cdot
\biggl(\frac{td}{4s^2}+1\biggr)^{-1}\cdot\frac{s\abs{N_i}}{d}
\ge \biggl(\frac{td}{8s^2}\biggr)^2\cdot
\biggl(\frac{3td}{8s^2}\biggr)^{-1}\cdot\frac{s\abs{N_i}}{d}
=  \frac{t\abs{N_i}}{24s}.
\end{equation*}
The lemma directly follows.
\end{proof}

\begin{lemma} \label{lm:xik_1}
For all $i \in [n]$,
\begin{equation*}
\PP\biggl(\frac{s}{t\abs{N_i}} \sum_{j \in N_i} \biggl(\frac{t}{s} - \x_j(k)\biggr)
\ge \frac{1}{4} \biggm\vert \x_i(k) = 1, \abs{N_i}\biggr)
\le \exp\biggl(-\frac{t\abs{N_i}}{640s}\biggr) + \exp\biggl(-\frac{s}{175}\biggr).
\end{equation*}
\end{lemma}

\begin{proof}
We consider a vertex $j \in N_i$.
Since $j \in N_i$, $\inner{\x_i}{\x_j} = \sum_{k \in S_i} \x_j(k) \ge t$.
Given $S_i$,
$\big(\x_j(k)\big)_{ k \in S_i}$ are identical (but not independent) Bernoulli random variables.
Hence $\PP(\x_j(k) = 1 \mid j \in N_i, k \in S_i) = \beta$ are the same for all $k \in S_i$.
Since
\begin{equation*}
\EE\biggl[\sum_{k \in S_i} \x_j(k) \biggm\vert {S_i},j \in N_i\biggr]
= \sum_{k \in S_i}
\PP(\x_j(k) = 1 \mid j \in N_i, k \in S_i) = \beta\abs{S_i} \ge t,
\end{equation*}
we have $\beta \ge t/\abs{S_i}$.
By Lemma~\ref{lm:S_lower} with $\delta = 1/4$, we have
\begin{equation}\label{TPD}
\PP\biggl(\abs{S_i} \le \frac{5}{4}s\biggr) \ge 1 - \exp\biggl(-\frac{s}{175}\biggr).
\end{equation}

Now conditioned on $\x_i(k) = 1$,
$\big(\x_j(k)\big)_{ j \in N_i}$ are independent Bernoulli random variables with parameter $\beta$.
By taking $\delta = 1 - \frac{3t}{4\beta}$ in Proposition~\ref{pp:chernoff},
we obtain that
\begin{equation*}
\begin{split}
&\PP\biggl(\frac{s}{t\abs{N_i}} \sum_{j \in N_i} \biggl(\frac{t}{s} - \x_j(k)\biggr)
\ge \frac{1}{4} \biggm\vert \x_i(k) = 1, \abs{N_i}, \abs{S_i}\biggr)\\
&\qquad= \PP\biggl(\sum_{j \in N_i} \x_j(k) \ge \frac{3t\abs{N_i}}{4s}
= \frac{3t}{4s\beta}\abs{N_i}\beta \biggm\vert \x_i(k) = 1, \abs{N_i}, \abs{S_i}\biggr)\\
&\qquad\le \exp\biggl(-\frac{1}{2}\biggl(1 - \frac{3t}{4s\beta}\biggr)^2\abs{N_i}\beta\biggr)
\le \exp\biggl(-\frac{1}{2}\biggl(1 - \frac{3\abs{S_i}}{4s}\biggr)^2\frac{t\abs{N_i}}{\abs{S_i}}\biggr).
\end{split}
\end{equation*}
Therefore,
\begin{equation*}
\PP\biggl(\frac{s}{t\abs{N_i}} \sum_{j \in N_i} \biggl(\frac{t}{s} - \x_j(k)\biggr)
\ge \frac{1}{4} \biggm\vert \x_i(k) = 1, \abs{N_i},
\abs{S_i} \ge \frac{5}{4}s\biggr)
\le \exp\biggl(-\frac{t\abs{N_i}}{640s}\biggr).
\end{equation*}
The claim directly follows from applying Proposition~\ref{pp:cond}.
\end{proof}

The following two lemmas are the major building blocks of our proof.
Intuitively, the probability that the neuron in GNN is activated when the true signal is $1$
or the neuron is deactivated when the true signal is $0$ is very small if the parameters
satisfy certain conditions.
\begin{lemma} \label{lm:A_0}
For all $k \in [d]$, we have when $n \ge 8^t m$,
\begin{equation*}
\PP(\cA_{i,k}^c \mid \x_i(k)=0, \abs{N_i}) \le \exp\biggl(-\frac{t\abs{N_i}}{24s}\biggr)
+\exp\biggl(-\frac{t\abs{N_i}}{32s^2\sigma^2}\biggr).
\end{equation*}
\end{lemma}

\begin{proof}
We further define the following two events:
\begin{equation}\label{def_b0_b1}
B_0 \coloneqq \biggl\{\frac{s}{t\abs{N_i}}
\sum_{j \in N_i} \x_j(k) \le \frac{1}{4}\biggr\}
\quad\textrm{and}\quad
B_1 \coloneqq \biggl\{\frac{s}{t\abs{N_i}}\sum_{j \in N_i} \beps_j(k) 
\le \frac{1}{4}\biggr\}.
\end{equation}
Then we have $\cA_{i,k} \supset B_0 \cap B_1$.
By a union bound, we obtain 
\begin{equation*}
\PP(\cA_{i,k}^c \mid \x_i(k) = 0, \abs{N_i})
\le \PP(B_0^c \mid \x_i(k) = 0, \abs{N_i})
+ \PP(B_1^c \mid \x_i(k)=0, \abs{N_i}).
\end{equation*}
Using Lemma~\ref{lm:xik_0}, we have
\begin{equation*}
\PP(B_0^c \mid \x_i(k) = 0, \abs{N_i})
\le \exp\biggl(-\frac{t\abs{N_i}}{24s}\biggr).
\end{equation*}
And Lemma~\ref{lm:eik} gives
\begin{equation*}
\PP(B_1^c \mid \x_i(k) = 0, \abs{N_i}) = \PP(B_1^c \mid \abs{N_i})
\le \exp\biggl(-\frac{t^2\abs{N_i}}{32s^2\sigma^2}\biggr).
\end{equation*}
The lemma is proved by combining the above displays.
\end{proof}

\begin{lemma} \label{lm:A_1}
For all $k \in [d]$,
\begin{equation*}
\PP(\cA_{i,k} \mid \x_i(k) = 1, \abs{N_i})
\le \exp\biggl(-\frac{t\abs{N_i}}{640s}\biggr) + \exp\biggl(-\frac{s}{175}\biggr)
+ \exp\biggl(-\frac{t^2\abs{N_i}}{32s^2\sigma^2}\biggr).
\end{equation*}
\end{lemma}

\begin{proof}
We similarly define the two events:
\begin{equation*}
B_0 \coloneqq \biggl\{\frac{s}{t\abs{N_i}}
\sum_{j \in N_i} \biggl(\frac{t}{s} - \x_j(k)\biggr) \le \frac{1}{4}\biggr\}
\quad\textrm{and}\quad
B_1 \coloneqq \biggl\{-\frac{s}{t\abs{N_i}}\sum_{j \in N_i} \beps_j(k) 
\le \frac{1}{4}\biggr\}.
\end{equation*}
Rearranging the terms, we have $\cA_{i,k}^c \supset B_0 \cap B_1$.
Hence by a union bound
\begin{equation*}
\PP(\cA_{i,k} \mid \x_i(k) = 1, \abs{N_i})
\le \PP(B_0^c \mid \x_i(k) = 1, \abs{N_i}) + \PP(B_1^c \mid \x_i(k) = 1, \abs{N_i}).
\end{equation*}
Using Lemma~\ref{lm:xik_1}, we arrive at
\begin{equation*}
\PP(B_0^c \mid \x_i(k) = 1, \abs{N_i})
\le \exp\biggl(-\frac{t\abs{N_i}}{640s}\biggr) + \exp\biggl(-\frac{s}{175}\biggr).
\end{equation*}
And applying Lemma~\ref{lm:eik} to $-\beps_i$'s we have
\begin{equation*}
\PP(B_1^c \mid \x_i(k) = 1, \abs{N_i})
\le \exp\biggl(-\frac{t^2\abs{N_i}}{32s^2\sigma^2}\biggr).
\end{equation*}
We obtain the claim by putting together the above displays.
\end{proof}

The next lemma combines the previous two and proves an upper bound
for making a mistake in the vertex features.
\begin{lemma} \label{lm:zi_xi}
For all $i \in [n]$ we have
\begin{equation*}
\PP(\z_i \ne \x_i)
\le 3d\exp\biggl(-\frac{mq}{5\cdot2^{t+7}s}\biggr)
+ 3d\exp\biggl(-\frac{s}{175}\biggr)
+ d\exp\biggl(-\frac{mq}{2^{t+7}s^2\sigma^2}\biggr).
\end{equation*}
\end{lemma}

\begin{proof}
We first note that by a union bound
\begin{equation*}
\PP(\z_i \ne \x_i) \le \sum_{i=1}^d \PP(\z_i(k) \ne \x_i(k)).
\end{equation*}
By the law of total probability, we have
\begin{equation*}
\begin{split}
\PP(\z_i(k) \ne \x_i(k))
&= \PP(\z_i(k) = 1 \mid \x_i(k) = 0) \times \PP(\x_i(k) = 0)\\
&\phantom{{}={}}+ \PP(\z_i(k) = 0 \mid \x_i(k) = 1) \times \PP(\x_i(k) = 1))\\
&\le \max\{\PP(\z_i(k) = 1 \mid \x_i(k) = 0), \PP(\z_i(k) = 0 \mid \x_i(k) = 1)\}
\end{split}
\end{equation*}

By Proposition~\ref{pp:cond}, using Lemma~\ref{lm:A_0} and~\ref{lm:N_i},
we obtain that
\begin{equation*}
\begin{split}
\PP(\z_i(k) = 1 \mid \x_i(k) = 0)
&\le \PP\biggl(\z_i(k) = 1 \biggm\vert \x_i(k) = 0,
\abs{N_i} \ge \frac{mq}{2^{t+2}} \biggr) 
+ \PP\biggl(\abs{N_i} \ge \frac{mq}{2^{t+2}}\biggr)\\
&\le 3\exp\biggl(-\frac{mq}{2^{t+7}s}\biggr)
+ \exp\biggl(-\frac{mq}{2^{t+7}s^2\sigma^2}\biggr)
+ 2\exp\biggl(-\frac{s}{8}\biggr).
\end{split}
\end{equation*}
Similarly with Lemma~\ref{lm:A_1},
\begin{equation*}
\begin{split}
\PP(\z_i(k) = 0 \mid \x_i(k) = 1)
&\le \PP\biggl(\z_i(k) = 0 \biggm\vert \x_i(k) = 1,
\abs{N_i} \ge \frac{mq}{2^{t+2}} \biggr) 
+ \PP\biggl(\abs{N_i} \ge \frac{mq}{2^{t+2}}\biggr)\\
&\le 3\exp\biggl(-\frac{tmq}{5\cdot2^{t+7}s}\biggr)
+ 3\exp\biggl(-\frac{s}{175}\biggr)
+ \exp\biggl(-\frac{t^2mq}{2^{t+7}s^2\sigma^2}\biggr).
\end{split}
\end{equation*}
Therefore, by combining the above two displays, we obtain that
\begin{equation*}
\PP(\z_i(k) \ne \x_i(k))
\le 3\exp\biggl(-\frac{mq}{5\cdot2^{t+7}s}\biggr)
+ 3\exp\biggl(-\frac{s}{175}\biggr)
+ \exp\biggl(-\frac{mq}{2^{t+7}s^2\sigma^2}\biggr).
\end{equation*}
The claim directly follows.
\end{proof}

With Lemma~\ref{lm:zi_xi} in place,
Theorem~\ref{th:perf} directly follows from a union bound:
\begin{equation*}
\begin{split}
\PP(\exists i\ \textrm{s.t.}\ \x_i \ne \z_i)
&\le \sum_{i=1}^n \PP(\x_i \ne \z_i)\\
&\le nd\biggl(3\exp\biggl(-\frac{mq}{5\cdot2^{t+7}s}\biggr)
+ 3\exp\biggl(-\frac{s}{175}\biggr)
+ \exp\biggl(-\frac{mq}{2^{t+7}s^2\sigma^2}\biggr)\biggr).
\end{split}
\end{equation*}
This concludes our proof for the perfect recovery.

\section{PROOFS OF THE IMPOSSIBILITY RESULTS FOR GRAPH NEURAL NETWORKS}
\begin{lemma} \label{lm:zik_low}
For all vertices $i \in [n]$,
\begin{equation*}
\PP(\z_i(k) = 1 \mid \x_i(k) = 0, \X, \abs{N_i})
\ge \Phi^c\biggl(\frac{t \sqrt{\abs{N_i}}}{2s \sigma}\biggr)
\end{equation*}
where $\Phi^c(x) \coloneqq 1 - \Phi(x)$ is the upper tail of standard normal distribution.
\end{lemma}

\begin{proof}
Define $\cA_{i,k}$ as in \eqref{eq:Aik} by \begin{equation*} 
\cA_{i,k} \coloneqq \biggl\{\frac{s}{t\abs{N_i}} \sum_{j \in N_i} \x_j(k)
+ \frac{s}{t\abs{N_i}}\sum_{j \in N_i} \beps_j(k)
\le \frac{1}{2} \biggr\}.
\end{equation*}
We remark that
\begin{equation*}
\begin{split}
&\PP(\z_i(k) = 1 \mid \x_i(k)=0, \X, \abs{N_i}) = \PP(\cA_{i,k}^c \mid \x_i(k)=0, \X, \abs{N_i})\\
&\qquad= \PP\biggl(\frac{s}{t\abs{N_i}} \sum_{j \in N_i} \x_j(k)
+ \frac{s}{t\abs{N_i}}\sum_{j \in N_i} \beps_j(k) \ge \frac{1}{2}
\biggm\vert \x_i(k) = 0, \X, \abs{N_i}\biggr)\\
&\qquad\overset{(a)}{\ge} \PP\biggl(\frac{s}{t\abs{N_i}}\sum_{j \in N_i} \beps_j(k)
\ge \frac{1}{2} \biggm\vert \x_i(k) = 0, \X, \abs{N_i} \biggr)\\
&\qquad\overset{(b)}{=} \PP\biggl(\frac{s}{t\abs{N_i}}\sum_{j \in N_i} \beps_j(k)
\ge \frac{1}{2} \biggm\vert \abs{N_i} \biggr)
\overset{}{=} \Phi^c\biggl(\frac{t \sqrt{\abs{N_i}}}{2s \sigma}\biggr),
\end{split}
\end{equation*}
where $(a)$ is due to the fact that
$\frac{s}{t\abs{N_i}} \sum_{j \in N_i} \x_j(k)\ge 0$ is nonnegative
and $(b)$ is by independence of the noise.
\end{proof}

\begin{lemma} \label{lm:del_pair}
For a pair of vertices $i$ and $j$,
\begin{equation*}
\begin{split}
&\PP(\bdel_i(k) - \bdel_j(k) = 1, \bdel_i'(k) - \bdel_j'(k) = 1
\mid \cM_{i,j}^k, \abs{N_i}, \abs{N_j}, \abs{N_i'}, \abs{N_j'})\\
&\qquad\ge \frac{1}{4}\Phi^c\biggl(\frac{t \sqrt{\abs{N_i}}}{2s \sigma}\biggr)
\Phi^c\biggl(\frac{t \sqrt{\abs{N_i'}}}{2s \sigma}\biggr)
\biggl(1 - \exp\biggl(-\frac{t\abs{N_j}}{24s}\biggr)
- \exp\biggl(-\frac{t\abs{N_j'}}{24s}\biggr)\biggr).
\end{split}
\end{equation*}
\end{lemma}

\begin{proof}
By conditional independence of the random variables,
\begin{equation*}
\begin{split}
&\PP(\bdel_i(k) - \bdel_j(k) = 1, \bdel_i'(k) - \bdel_j'(k) = 1
\mid \cM_{i,j}^k, \abs{N_i}, \abs{N_j}, \abs{N_i'}, \abs{N_j'})\\
&\qquad= \EE_{\X}[\PP(\bdel_i(k) = 1 \mid \cM_{i,j}^k, \X, \abs{N_i})
\PP(\bdel_i'(k) = 1 \mid \cM_{i,j}^k, \X, \abs{N_i'})\\
&\qquad\phantom{{}= \EE_{\X}[}
\times \PP(\bdel_j(k) = 0, \bdel_j'(k) = 0 \mid \cM_{i,j}^k, \X, \abs{N_j}, \abs{N_j'})]\\
&\qquad\ge \Phi^c\biggl(\frac{t \sqrt{\abs{N_i}}}{2s \sigma}\biggr)
\Phi^c\biggl(\frac{t \sqrt{\abs{N_i'}}}{2s \sigma}\biggr)
\PP(\bdel_j(k) = 0, \bdel_j'(k) = 0 \mid \cM_{i,j}^k, \abs{N_j}, \abs{N_j'})
\end{split}
\end{equation*}
where we used Lemma~\ref{lm:zik_low} in the inequality.
\begin{equation*}
\begin{split}
&\PP(\bdel_j(k) = 0, \bdel_j'(k) = 0 \mid \cM_{i,j}^k, \abs{N_j}, \abs{N_j'})\\
&\qquad\ge \PP\biggl(\frac{s}{t\abs{N_j}} \sum_{l \in N_j} \x_{l}(k) \le \frac{1}{4},
\frac{s}{t\abs{N_j'}} \sum_{l \in N_j'} \x_l(k) \le \frac{1}{4},\\
&\qquad\phantom{{}\ge \PP\biggl(}\frac{s}{t\abs{N_j}}\sum_{l \in N_j'} \beps_l(k) \le \frac{1}{4},
\frac{s}{t\abs{N_j'}} \sum_{l \in N_j'} \beps_l'(k) \le \frac{1}{4}
\biggm\vert \cM_{i,j}^k, \abs{N_j}, \abs{N_j'}\biggr)\\
&\qquad= \PP\biggl(\frac{s}{t\abs{N_j}}\sum_{l \in N_j} \beps_l(k) \le \frac{1}{4}
\biggm\vert \abs{N_j} \biggr)
\PP\biggl(\frac{s}{t\abs{N_j'}} \sum_{l \in N_j'} \beps_l'(k) \le \frac{1}{4}
\biggm\vert \abs{N_j} \biggr)\\
&\qquad\phantom{{}={}}
\times \PP\biggl(\frac{s}{t\abs{N_j}} \sum_{l \in N_j} \x_{l}(k) \le \frac{1}{4},
\frac{s}{t\abs{N_j'}} \sum_{l \in N_j'} \x_l(k) \le \frac{1}{4}
\biggm\vert \x_j(k) = 0, \abs{N_j}, \abs{N_j'}\biggr).
\end{split}
\end{equation*}
Since $\frac{s}{t\abs{N_j}}\sum_{l \in N_j} \beps_l(k)$ is a sum of independent centered Gaussian
random variables,
\begin{equation*}
\PP\biggl(\frac{s}{t\abs{N_j}}\sum_{l \in N_j} \beps_l(k) \le \frac{1}{4}
\biggm\vert \abs{N_j} \biggr)
\ge \PP\biggl(\frac{s}{t\abs{N_j}}\sum_{l \in N_j} \beps_l(k) \le 0
\biggm\vert \abs{N_j} \biggr) = \frac{1}{2}.
\end{equation*}
The same holds for $\frac{s}{t\abs{N_j'}} \sum_{l \in N_j'} \beps_l'(k)$.
By a union bound and Lemma~\ref{lm:xik_0},
\begin{equation*}
\begin{split}
&\PP\biggl(\frac{s}{t\abs{N_j}} \sum_{l \in N_j} \x_{l}(k) \le \frac{1}{4},
\frac{s}{t\abs{N_j'}} \sum_{l \in N_j'} \x_l(k) \le \frac{1}{4}
\biggm\vert \x_j(k) = 0, \abs{N_j}, \abs{N_j'}\biggr)\\
&\qquad\ge 1 -
\PP\biggl(\frac{s}{t\abs{N_j}} \sum_{l \in N_j} \x_{l}(k) \ge \frac{1}{4}
\biggm\vert \x_j(k) = 0, \abs{N_j}\biggr)\\
&\qquad\phantom{{}\ge{}}- \PP\biggl(\frac{s}{t\abs{N_j'}} \sum_{l \in N_j'} \x_l(k) \ge \frac{1}{4}
\biggm\vert \x_j(k) = 0, \abs{N_j'}\biggr)\\
&\qquad\ge 1 - \exp\biggl(-\frac{t\abs{N_j}}{24s}\biggr)
- \exp\biggl(-\frac{t\abs{N_j'}}{24s}\biggr).
\end{split}
\end{equation*}
Putting the above displays together, we hence proved the lemma.
\end{proof}

Denote $i \bowtie j$ the event
\begin{equation*}
\{\inner{\z_i}{\z_j'} + \inner{\z_j}{\z_i'}
\ge \inner{\z_i}{\z_i'} + \inner{\z_j}{\z_j'}\}.
\end{equation*}
It is clear that when $i \bowtie j$ happens,
perfect recovery is not possible since by swapping $i$ and $j$,
the loss $\ell$ in \eqref{eq:mp_align} is always not bigger.
The following proposition shows this and
therefore concludes the proof for impossibility of perfect recovery.

\begin{proposition} \label{pp:bowtie}
Assume that $\min\{s, \frac{qm}{s}\} \ge c$
and $\frac{qm}{\sigma^2s^2} \le C$ for constants $c, C > 0$.
Then there is a constant $\delta(c, C) > 0$ such that for all $\delta' \in [0, \delta(c, C)]$
\begin{equation*}
\PP(i \bowtie j) \ge \delta'.
\end{equation*}
\end{proposition}

\begin{proof}
The definition of $i \bowtie j$ reads
\begin{equation*}
\inner{\z_i}{\z_i'} + \inner{\z_j}{\z_j'}
\le \inner{\z_i}{\z_j'} + \inner{\z_j}{\z_i'}
\end{equation*}
which by definition of $\z_i,\z_j$ and of $\z'_i,\z'_j$ means that
\begin{equation*}
\inner{\x_i + \bdel_i}{\x_i + \bdel_i'}
+ \inner{\x_j + \bdel_j}{\x_j + \bdel_j'}
\le \inner{\x_i + \bdel_i}{\x_j + \bdel_j'}
+ \inner{\x_j + \bdel_j}{\x_i + \bdel_i'},
\end{equation*}where $\bdel_i \coloneqq \z_i - \x_i$ and $\bdel_j \coloneqq \z_j - \x_j$. 
Rearranging the terms, we have
\begin{equation*}
\norm{\x_i - \x_j}^2 + \inner{\x_i - \x_j}{\bdel_i - \bdel_j} + \inner{\x_i - \x_j}{\bdel_i' - \bdel_j'}
\le - \inner{\bdel_i - \bdel_j}{\bdel_i' - \bdel_j'}.
\end{equation*}
Observe that 
\begin{equation*}
\begin{split}
\inner{\bdel_i - \bdel_j}{\bdel_i' - \bdel_j'}
&= \underbrace{\sum_{k \in S_i \cup S_j} (\bdel_i(k) - \bdel_j(k))(\bdel_i'(k) - \bdel_j'(k))}_{D_0}\\
&\phantom{{}={}}+ \underbrace{\sum_{k \in (S_i \cup S_j)^c} (\bdel_i(k) - \bdel_j(k))(\bdel_i'(k) - \bdel_j'(k))}_{D_1}.
\end{split}
\end{equation*} We will bound both terms on the right-hand side. 
In this goal, we recall the event
$\cE_i = \{\frac{s}{2} \le \abs{S_i} \le 2s\}$ defined in \eqref{eq:Si_s}.
Using \eqref{eq:E_prob} we know that the event $\cE_i\cap \cE_j$
happens with probability at least $1-2e^{-s/8}$.
We notice that when $cE_i\cap \cE_j$ holds we have 
\begin{equation*}
\norm{\x_i - \x_j}^2 \le \abs{S_i} + \abs{S_j} \le 4s.
\end{equation*}
By definition of $\x_i$'s and $\z_i$'s, we know that
\begin{equation*}
\inner{\x_i - \x_j}{\bdel_i - \bdel_j} \le 2\sum_{k=1}^d\abs{\x_i(k) - \x_j(k)} \overset{(a)}{=} 2\norm{\x_i - \x_j}^2 \le 2(\abs{S_i} + \abs{S_j}) \le 8s
\end{equation*}
where to obtain $(a)$ we used the fact that
as $(\x_i(k))$ and $(\x_j(k))$ are Bernoulli random variables
$\norm{\x_i - \x_j}^2=\sum_{k=1}^d\abs{\x_i(k) - \x_j(k)}$.
Recall that $S_i$ is the support of $\x_i$.
First note that for every $k$ we have $\abs{\bdel_i(k) - \bdel_j(k)} \overset{a.s}{\le }2$.
Hence we obtain that on event $\cE_i \cap \cE_j$
\begin{equation*}
\abs{D_0 }\le 4 \abs{S_i \cup S_j} \le 4(\abs{S_i} + \abs{S_j}) \le 16s.
\end{equation*}
Putting the above together, this implies that when $\cE_i\cap \cE_j$ holds we have 
\begin{equation} \label{eq:bowtie_cross}
\bigl\lvert \norm{\x_i - \x_j}^2 + \inner{\x_i - \x_j}{\bdel_i - \bdel_j} + \inner{\x_i - \x_j}{\bdel_i' - \bdel_j'} + D_0\bigr\rvert \le 28s.
\end{equation}

We now aim at bounding $D_1$.
In this goal, we let $\b \in \{-1, 0, 1\}^d$ be the vector such that
\begin{equation*}
\b(k) \coloneqq
\begin{cases}
(\bdel_i(k) - \bdel_j(k))(\bdel_i'(k) - \bdel_j'(k))
&\textrm{if}\ k \in (S_1 \cup S_2)^c.\\
0 &\textrm{otherwise.}
\end{cases}
\end{equation*}
Moreover for $k \in (S_1 \cup S_2)^c$, by definition we have $\x_i(k) = \x_j(k) = 0$. 
For ease of notation we define the events
\begin{equation*}
\cM^k_{i,j} \coloneqq \{\x_i(k) = \x_j(k) =0\}.
\end{equation*}
We will first see that $\PP(\b(k) = 1 \mid \cM_{i,j}^k) = \PP(\b(k) = -1 \mid \cM_{i,j}^k)$
for all $k \in (S_1\cup S_2)^c$.
We will then lower bound the probability of $\b(k)=\pm 1$ conditioned on a few events. 

Recall the definition of $\cE_i, \cE_j$ in \eqref{eq:Si_s}
and the events defined in \eqref{eq:Fi}, and let
\begin{equation*}
\cF_i \coloneqq \biggl\{\frac{1}{2^{t+2}}\cdot mq\le
\abs{N_i} \le 2^{t+2}e^t \cdot mq\biggr\}
\quad\textrm{and}\quad
\cF'_i \coloneqq \biggl\{\frac{1}{2^{t+2}}\cdot mq\le
\abs{N'_i} \le 2^{t+2}e^t \cdot mq\biggr\}.
\end{equation*}
Further let
$\cH_{i,j} \coloneqq \cE_i \cap \cE_j \cap \cF_i \cap \cF_j \cap \cF_i' \cap \cF_j'$.

As conditionally on $\cM_{i,j}^k \cap \cH_{i,j}$
the laws of $\bdel_i(k),\bdel_j(k) \in \{0, 1\}$ are the same,
by symmetry we have
\begin{equation*}
\PP(\bdel_i(k) - \bdel_j(k) = 1 \mid \cM_{i,j}^k, \cH_{i,j})\\
= \PP(\bdel_i(k) - \bdel_j(k) = -1 \mid \cM_{i,j}^k, \cH_{i,j}).
\end{equation*}
By using Lemma~\ref{lm:del_pair}, we obtain that
\begin{equation*}
\begin{split}
p &\coloneqq \PP(\b(k) = 1 \mid \cM_{i,j}^k, \cH_{i,j})\\
&= \PP(\bdel_i(k) - \bdel_j(k) = 1,
\bdel_i'(k) - \bdel_j'(k) = 1 \mid \cM_{i,j}^k, \cH_{i,j})\\
&\phantom{{}={}}+ \PP(\bdel_i(k) - \bdel_j(k) = -1,
\bdel_i'(k) - \bdel_j'(k) = -1 \mid \cM_{i,j}^k, \cH_{i,j})\\
&\ge \frac{1}{2}\Phi^c\biggl(\frac{c_t\sqrt{mq}}{s\sigma}\biggr)^2
\biggl(1 - 2\exp\biggl(-\frac{C_t mq}{s}\biggr)\biggr)
\end{split}
\end{equation*}
for constants $c_t, C_t > 0$ that only depend on $t$.
Since for $k \in (S_i \cup S_j)^c$
the entries $\x_{i'}(k), i' \in N_i$ is independent of $N_i$ hence
$\z_i(k)$ only depend on the size of the neighborhood.
Therefore, $\b(k)$'s are independent of each other.
Applying Proposition~\ref{pp:lazy_tail} to $\sum_{k \in (S_1 \cup S_2)^c} \b(k)$ we obtain that
\begin{equation*}
\PP\biggl(\sum_{k \in (S_1 \cup S_2)^c} \b(k) \ge 28s
\biggm\vert \cH_{i,j}\biggr)
\ge c_{\beta} \exp\biggl(-C_\beta \frac{s^2}{p (d - 4s)}\biggr).
\end{equation*}
Further use the bound \eqref{eq:bowtie_cross}, we have that
\begin{equation*}
\begin{split}
\PP(i \bowtie j)
&= \PP(i \bowtie j \mid \cH_{i,j}) \PP(\cH_{i,j})
\ge \PP\biggl(\sum_{k \in (S_1 \cup S_2)^c} \b(k) \ge 28s
\biggm\vert \cH_{i,j}\biggr) \PP(\cH_{i,j})\\
&\ge c_{\beta} \exp\biggl(-C_\beta \frac{s^2}{p (d - 4s)}\biggr)\PP(\cH_{i,j}).
\end{split}
\end{equation*}
By a union bound, we know that $\PP(\cH_{i,j})$ is bounded away from $0$
if for a constant $c > 0$,
\begin{equation*}
\min\{s, mq\} \ge c.
\end{equation*}
Therefore, $\PP(\bowtie)$ is bounded away from $0$
if $p$ is bounded away from $0$ which is satisfied if
\begin{equation*}
\frac{mq}{s} \ge c'\quad\textrm{and}\quad\frac{mq}{s^2\sigma^2} \le C
\end{equation*}
for constants $c', C > 0$.
The lemma is hence proved.
\end{proof}

\section{PROOFS OF THE IMPOSSIBILITY RESULTS FOR NOISY FEATURE ALIGNMENT}
In Proposition \ref{pp:low1} we will prove that
if $\sigma^2\ge4{s}\big({1+\frac{K}{n}}\big)^{-2}$
then the probability of perfect recovery is bounded away from $0$.
We state the following information-theoretic proposition
before we move to the proof.
\begin{proposition} \label{pp:YH_TV}
Let $k\le n$ be any arbitrary integer and $\pi \in \cP(n)$ be an arbitrary permutation.
Let $\pi(\Y) = (\y_{\pi(1)},\dots,\y_{\pi(k)})$ be the noisy observations of
$\pi(\X)=(\x_{\pi(1)},\dots,\x_{\pi(k)})$ as defined in Definition~\ref{df:nirig}.
Choose $\G=(\g_1,\dots,\g_{k})\in \RR^{k \times d}$ to be an independent Gaussian
random matrix with i.i.d.\ entries $\g_{i}(j) \sim \cN(0, \sigma^2)$. 
Then, the total variation distance satisfies
\begin{equation*}
\TV(\pi(\Y), \G \mid \X) \le \frac{1}{2\sigma} \norm{\vec(\X)}
\end{equation*}
where we denoted 
\begin{equation*}
\TV(\pi(\Y),\G \mid \X)
\coloneqq \sup_{A \in \mathcal{B}(\mathbb{R}^{n \times d})}
\lvert\PP(\pi(\Y) \in A \mid \X) - \PP(\G \in A)\rvert.
\end{equation*}
\end{proposition}

\begin{proof}
By Pinsker's inequality, we know that 
\begin{equation*}
\TV(\pi(\Y), \G \mid \X)
\le \sqrt{\frac{1}{2}\KL(\pi(\Y) \parallel \G \mid \X)}.
\end{equation*}
We remark that by definition of $\pi(\Y)$,
we have $\vec(\pi(\Y)) \sim \cN(\vec(\X_\pi), \sigma^2 \I_{kd})$.
Hence using the formula for the KL divergence between two Gaussian vectors
(see, e.g., \citep[Exercise 15.13(b)]{wainwright2019high}),
we have
\begin{equation*}
\KL(\pi(\Y) \parallel \G \mid \X)
= \frac{1}{2\sigma^2}\norm{\vec(\pi(\X))}^2
= \frac{1}{2\sigma^2}\norm{\vec(\X)}^2.
\end{equation*}
The claim is hence proved.
\end{proof}

With the previous proposition in place, we are ready to prove the results.
First denote by $i \bowtie j$ the following event:
\begin{equation*}
\{\inner{\y_i}{\y_j'} + \inner{\y_j}{\y_i'}
\ge \inner{\y_i}{\y_i'} + \inner{\y_j}{\y_j'}\}.
\end{equation*}
Note that $i \bowtie j$ is equivalent to
\begin{equation*}
\{\norm{\y_i - \y_j'}^2 + \norm{\y_i - \y_j'}^2
\le \norm{\y_i - \y_i'}^2 + \norm{\y_j - \y_j'}^2\}.
\end{equation*}
Hence when $i \bowtie j$ happens,
perfect recovery is not possible since if we swap $i$ and $j$,
the loss $\ell$ in \eqref{lin_algo} is not increased.

\begin{proposition} \label{pp:low1}
Let $\tilde{\pi}$ be the solution of \eqref{lin_algo}.
If $\sigma^2\ge2{s}\bigl({1+\frac{K}{n}}\bigr)^{-2}$ for any $K>4$ we have that
\begin{equation*}
\PP(\tilde{\pi} = \pi^*) \le e^{-\frac{K-4}{4}}.
\end{equation*}
Hence if $\sigma^2 \gg s$ we have that the probability of perfect recovery will converge to $0$.
\end{proposition}

\begin{proof}
First by triangle inequality and Proposition \ref{pp:YH_TV}
we have
\begin{equation*}
\begin{split}
&\TV((\y_1,\y_2),(\y_2,\y_1) \mid \x_{1:2})\\
&\qquad\le \TV((\y_1,\y_2),(\g_1,\g_2) \mid \x_{1:2})
+\TV((\y_2,\y_1),(\g_1,\g_2) \mid \x_{1:2})\\
&\qquad\le \frac{1}{\sigma}(\norm{\x_1}^2+\norm{\x_2}^2)^{1/2}.
\end{split}
\end{equation*}
This directly implies that if we write the following set
\begin{equation*}
A \coloneqq \{(\a,\b): \inner{\a}{\y_2'} + \inner{\b}{\y_1'}
\ge \inner{\a}{\y_1'} + \inner{\b}{\y_2'}\}
\end{equation*}
then we have
\begin{equation*}
\abs{\PP((\y_1,\y_2)\in A \mid \x_{1:2}) - \PP((\y_2,\y_1) \in A \mid \x_{1:2})}
\le \frac{1}{\sigma}(\norm{\x_1}^2+\norm{\x_2}^2)^{1/2}.
\end{equation*}
Now by Jensen inequality we have
\begin{equation} \label{eq:y12_bound}
\begin{split}
&\abs{\PP((\y_1,\y_2)\in A) - \PP((\y_2,\y_1) \in A)}\\
&\qquad\le \EE[\abs{\PP((\y_1,\y_2)\in A \mid \x_{1:2}) - \PP((\y_2,\y_1) \in A \mid \x_{1:2})}]\\
&\qquad\le \frac{1}{\sigma} \EE[(\norm{\x_1}^2+\norm{\x_2}^2)^{1/2}]
\le \frac{1}{\sigma} \sqrt{\EE[\norm{\x_1}^2] + \EE[\norm{\x_2}^2]}
\overset{(a)}{=} \frac{\sqrt{2s}}{\sigma}
\end{split}
\end{equation}
where to get (a) we used the fact that $\|\x_1\|^2,\|\x_2\|^2 \sim \Binom(d,\frac{s}{d}).$

Now note that for continuous random variables $(\y_1,\y_2)$ and $(\y'_1, \y'_2)$,
\begin{equation*}
\PP(\inner{\y_1}{\y_2'} + \inner{\y_2}{\y_1'} = \inner{\y_1}{\y_1'} + \inner{\y_2}{\y_2'}) = 0 
\end{equation*}
The events $\{(\y_1,\y_2)\in A\}$ and $\{(\y_2,\y_1)\in A\}$ are disjoint
except for a zero-probability event.
This implies that 
\begin{equation*}
\begin{split}
1&= \PP((\y_1,\y_2) \in A) + \PP((\y_2,\y_1) \in A)
\overset{(a)}{\le} 2\PP((\y_1,\y_2) \in A) + \frac{\sqrt{2s}}{\sigma}
\end{split}
\end{equation*}
where $(a)$ is a consequence of \eqref{eq:y12_bound}.
Hence when $\sigma^2 \ge {2s}\bigl({1+\frac{K}{n}}\bigr)^{-2}$
for some $\delta > 0$,
we obtain that
\begin{equation*}
\PP((\y_1,\y_2)\in A) \ge \frac{K}{2n}.
\end{equation*}
This implies that
\begin{equation*}
\PP(\tilde\pi(1) \ne 1\;\textrm{or}\;\tilde\pi(2)\ne 2)
\ge\frac{K}{2n}.
\end{equation*}
Using the fact that the random variables
$(\bbmone\{\tilde\pi(i) \ne i\;\textrm{or}\;\tilde\pi(i+1) \ne i+1\})
_{\substack{i~\textrm{odd}\\i\le n-1}}$ are independent we obtain that 
\begin{equation*}
\begin{split}
\PP(\exists i\;\textrm{s.t.}\;\tilde{\pi}(i)\ne i)
&= \PP(\exists i\;\textrm{odd s.t.}\;\tilde{\pi}(i)\ne i \;\textrm{or}\; \tilde{\pi}(i+1) \ne i+1)\\
&\ge 1 - \biggl(1-\frac{K}{2n}\biggr)^{n/2-2}
\ge 1 - \exp\biggl(-\frac{K-4}{4}\biggr).
\end{split}
\end{equation*}
Hence we proved the desired result.
\end{proof}
The following proposition complements the previous proposition
and shows that when $\sigma^2 \le s$ perfect recovery is still not possible
if $\frac{\sigma^2\sqrt{d}}{s}$ is sufficiently large.

\begin{proposition} \label{pp:low2}
Suppose that $\frac{1}{4}\sigma^2\le s\le \frac{d}{2}$.
Assume that $\tilde\pi$ is the solution of \eqref{lin_algo}.
If $\sigma^2 \ge \frac{\sqrt{21}s}{\sqrt{d}}$ and then
\begin{equation*}
\PP(\tilde{\pi} \ne \pi^*) \ge \frac{1}{8}.
\end{equation*}
If $s \gg b_n$ and $\sigma^2\gg\frac{s}{\sqrt{db_n}}$ for a sequence $(b_n)$
diverging to infinity at a rate slower than $b_n=O(\log n)$
then
\begin{equation*}
\lim_{n\to\infty} \PP(\tilde{\pi} = \pi^*) = 0.
\end{equation*}
\end{proposition}

\begin{proof}
Recall that $i \bowtie j$ if
\begin{equation*}
\inner{\y_i}{\y_i'} + \inner{\y_j}{\y_j'}
\le \inner{\y_i}{\y_j'} + \inner{\y_j}{\y_i'};
\end{equation*}
which by definition of $\y_i,\y_j$ and of $\y'_i,\y'_j$ means that
\begin{equation*}
\inner{\x_i + \beps_i}{\x_i + \beps_i'}
+ \inner{\x_j + \beps_j}{\x_j + \beps_j'}
\le \inner{\x_i + \beps_i}{\x_j + \beps_j'}
+ \inner{\x_j + \beps_j}{\x_i + \beps_i'}.
\end{equation*}
Rearranging the terms yields $i \bowtie j$ if and only if 
\begin{equation*}
\norm{\x_i - \x_j}^2
+ \inner{\x_i - \x_j}{\beps_i - \beps_j + \beps_i' - \beps_j'}
\le - \inner{\beps_i - \beps_j}{\beps_i' - \beps_j'}.
\end{equation*}
Define the events
\begin{align}\label{ts11}
A_0 &\coloneqq \biggl\{- \inner{\beps_i - \beps_j}{\beps_i' - \beps_j'}
\ge (a_1 + a_2)s\biggr\},\\
A_1 &\coloneqq \biggl\{\norm{\x_i - \x_j}^2 \le a_1 s\biggr\}\nonumber,\\
A_2 &\coloneqq \biggl\{\inner{\x_i - \x_j}
{\beps_i - \beps_j + \beps_i' - \beps_j'}
\le a_2 s\biggr\}.\nonumber
\end{align}

We first remark that 
\begin{equation*}
\begin{split}
\PP(i \bowtie j)
&\ge \PP(A_0 \cap A_1 \cap A_2)
= \PP(A_0) - \PP(A_0 \cap A_1^c) - \PP(A_0 \cap A_1 \cap A_2^c)\\
&\ge \PP(A_0) - \PP(A_1^c) - \PP(A_1 \cap A_2^c)
\ge \PP(A_0) - \PP(A_1^c) - \PP(A_2^c \mid A_1).
\end{split}
\end{equation*}
Hence to obtain the desired result we will bound $\PP(A_0)$ from below
and $\PP(A_1^c),\PP(A_2^c \mid A_1)$ from above.

We first focus on the first result. 
In this goal, we remark that
$\inner{\beps_i - \beps_j}{\beps_i' - \beps_j'}
= \sum_{k=1}^d(\beps_i(k) - \beps_j(k))(\beps_i'(k) - \beps_j'(k))$
is a sum of i.i.d.\ random variables.
Hence to lower bound $\PP(A_0)$ we can use the central limit theorem.
We choose $a_1 = 7$ and $a_2 = 14$ in \eqref{ts11}. Moreover, for all  $k \in [d]$ we define 
\begin{equation*}
\b(k) \coloneqq (\beps_i(k) - \beps_j(k))
(\beps_i'(k) - \beps_j'(k)).
\end{equation*}

By definition of $\beps_i$ and $\beps_j$, we know that
\begin{align*}
\EE[\b(k)] &= 0,\\
\EE[\b(k)^2] &= \EE[(\beps_i(k) - \beps_j(k))^2]
\EE[(\beps_i'(k) - \beps_j'(k))^2] = 4\sigma^4,\\
\EE[\abs{\b(k)}^3] &= \EE[\abs{\beps_i(k) - \beps_j(k)}^3]
\EE[\abs{\beps_i'(k) - \beps_j'(k)}^3] = \frac{64\sigma^6}{\pi}.
\end{align*}
Therefore, by Berry--Esseen theorem~\citep{tyurin2010improvement}, we have
\begin{equation*}
\biggl\lvert\PP(A_0) - \Phi^c\biggl(
\frac{21s}{2\sigma^2\sqrt{d}}\biggr)\biggr\rvert
\le \frac{4}{\pi\sqrt{d}}.
\end{equation*}
Since $d \ge (21s/\sigma^2)^2$ and $s \ge 4\sigma^2$, we obtain that 
\begin{equation*}
\PP(A_0)
\ge \Phi^c\biggl(\frac{21s}{2\sigma^2\sqrt{d}}\biggr) - \frac{4}{\pi\sqrt{d}}
\ge \Phi^c\biggl(\frac{1}{4}\biggr) - \frac{4}{42\pi} \ge \frac{3}{10} - \frac{1}{30} = \frac{4}{15}.
\end{equation*}

\noindent We now aim to upper bound $\PP(A_1^c)$.
Since
\begin{equation*}
\EE\norm{\x_i - \x_j}^2 = 2d\cdot\frac{s}{d}\cdot\biggl(1 - \frac{s}{d}\biggr)
= 2s\biggl(1 - \frac{s}{d}\biggr),
\end{equation*}
using $\delta = 6$ in Proposition~\ref{pp:chernoff} yields
\begin{equation*}
\PP(A_1^c)
\le \exp\biggl(-\frac{\delta^2 \EE\norm{\x_i - \x_j}^2}{2+\delta}\biggr)
= \exp\biggl(- 9s\biggl(1 - \frac{s}{d}\biggr)\biggr) \le \exp\biggl(-\frac{9}{2}\biggr)
\end{equation*}
for $d \ge 2s$.

Finally we remark that conditional on $\x_i$ and $\x_j$, the random variable
$\inner{\x_i - \x_j}{\beps_i - \beps_j + \beps_i' - \beps_j'}$
is distributed as $\cN(0, 4\norm{\x_i - \x_j}^2\sigma^2)$.
Hence, by Gaussian concentration, we obtain that 
\begin{equation*}
\PP(A_2^c \mid \x_i, \x_j)
\le \exp\biggl(-\frac{49s^2}{2\norm{\x_i - \x_j}^2\sigma^2}\biggr)
\le \exp\biggl(-\frac{98s}{\norm{\x_i - \x_j}^2}\biggr).
\end{equation*}
Then, we have
\begin{equation*}
\PP(A_2^c \mid A_1) \le \exp(-7).
\end{equation*}

Putting them together, we have when $d \ge{(21s)^2/\sigma^4}$
\begin{equation*}
\PP(i \bowtie j) \ge \frac{4}{15} - \exp\biggl(-\frac{9}{2}\biggr) - \exp(-7) \ge \frac{1}{4}.
\end{equation*}
Now by a union bound argument we know that
\begin{equation*}
\PP(i \bowtie j) \le \PP(\hat\pi(i)\ne i)+\PP(\hat\pi(j)\ne j)
= 2\PP(\hat\pi(i)\ne i).
\end{equation*}
This directly implies the first result.

We will now establish the second one.
In this goal we note that as we have assumed that $s\gg b_n$ we know that
\begin{align*}
\max\biggl(\frac{\sqrt{sb_n}}{16(1+\sqrt{s/b_n})},\frac{\sqrt{sb_n}}{2(2\sqrt{s/b_n}+1)}\biggr)
&\sim \max\biggl( \frac{b_n}{16},\frac{b_n}{4} \biggr),\\
\frac{11(1+\sigma+\sqrt{s/b_n})^2sb_n}{2d\sigma^4}
&\sim \frac{11s^2}{2d\sigma^4}.
\end{align*}
Moreover we have assumed that $b_n\gg \frac{s^2}{d\sigma^4}$ 
hence there exist constants $C_1,C_2>0$ such that for $n$ large enough
we have 
\begin{align*}
&\exp\biggl(-\frac{C_2^2\sqrt{sb_n}}{16(C_1+\sqrt{s/b_n})}\biggr)
+\exp\biggl(-\frac{\sqrt{sb_n}C_1^2}{2(2\sqrt{s/b_n}+C_1)}\biggr)\\
&\qquad\ge\frac{1}{2}\exp\biggl(-\frac{11(C_1+C_2\sigma+\sqrt{s/b_n})^2sb_n}{2d\sigma^4}\biggr)
\end{align*}
and
\begin{align*}
&\exp\biggl(-\frac{11(C_1+C_2\sigma+\sqrt{s/b_n})^2sb_n}{2d\sigma^4}\biggr)
\le 1-\frac{\sqrt{3}}{2}.
\end{align*}
Set $a_1 \coloneqq C_1\frac{\sqrt{b_n}}{\sqrt{s}}+1$ and $a_2 \coloneqq C_2\sigma\frac{\sqrt{b_n}}{\sqrt{s}}$.
Using Proposition~\ref{pr:inner_lower} with $u=(a_1+a_2)s/\sqrt{d}$
we have
\begin{align*}
\PP(A_0) &\ge \biggl(1 - \exp\biggl({-\frac{(a_1+a_2)^2s^2}{d\sigma^4}}\biggr)\biggr)^2
\exp\biggl({-\frac{11(a_1+a_2)^2s^2}{2d\sigma^4}}\biggr)\\
&\ge  \biggl(1 - \exp\biggl({-\frac{(C_1+C_2\sigma+\sqrt{{s}/b_n})^2sb_n}{d\sigma^4}}\biggr)\biggr)^2 
\exp\biggl({-\frac{11(C_1+C_2\sigma+\sqrt{s/b_n})^2sb_n}{2d\sigma^4}}\biggr)\\
&\sim\biggl(1 - \exp\biggl({-\frac{s^2}{d\sigma^4}}\biggr)\biggr)^2
\exp\biggl({-\frac{11s^2}{2d\sigma^4}}\biggr).
\end{align*}
Write $\delta=a_1-1$  then we have\begin{align*}
\PP(A_1^c)
&\le \exp\biggl(-\frac{\delta^2 \EE\norm{\x_i - \x_j}^2}{2+\delta}\biggr)
= \exp\biggl(- \frac{(a_1-1)^2}{1+a_1}s\biggl(1 - \frac{s}{d}\biggr)\biggr)\\
&\overset{(a)}{\le}\exp\biggl(- \frac{(a_1-1)^2}{2(1+a_1)}s\biggr)
\le \exp\biggl(- \frac{\sqrt{sb_n}C_1^2}{2(2\sqrt{s/b_n}+C_1)}\biggr)
\sim  \exp\biggl(- \frac{b_nC_1^2}{4}\biggr)
\end{align*}
where (a) is due to the assumption that $d\ge 2s$. Now once again, conditionally on $\x_i$ and $\x_j$, the random variable
$\inner{\x_i - \x_j}{\beps_i - \beps_j + \beps_i' - \beps_j'}$
is distributed as $\cN(0, 4\norm{\x_i - \x_j}^2\sigma^2)$.
Hence we also have \begin{equation*}
\PP(A_2^c \mid \x_i, \x_j)
\le \exp\biggl(-\frac{a_2^2s^2}{16\norm{\x_i - \x_j}^2\sigma^2}\biggr)
\end{equation*}
and therefore
\begin{equation*}
\PP(A_2^c \mid A_1)
\le\exp\biggl(-\frac{a_2^2s}{16a_1\sigma^2}\biggr)\le \exp\biggl(-\frac{C_2^2\sqrt{sb_n}}{16(C_1+\sqrt{s/b_n})}\biggr)\sim\exp\biggl(-\frac{C_2^2b_n}{16}\biggr) .
\end{equation*}
Hence, for $n$ large enough we have
\begin{equation*}
\begin{split}
\PP(i \bowtie j)
&\ge \exp\biggl({-\frac{11(C_1+C_2\sigma+\sqrt{s/b_n})^2sb_n}{2d\sigma^4}}\biggr)\\
&\phantom{{}\ge{}}
\times\biggl(\biggl(1 - \exp\biggl({-\frac{(C_1+C_2\sigma+\sqrt{{s}/b_n})^2sb_n}{d\sigma^4}}\biggr)\biggr)^2
-\frac{1}{2}\biggr)
\\&\ge \frac{1}{4}\exp\biggl({-\frac{11(C_1+C_2\sigma+\sqrt{s/b_n})^2sb_n}{2d\sigma^4}}\biggr)
\sim\frac{1}{4}\exp\biggl({-\frac{11s^2}{2d\sigma^4}}\biggr).
\end{split}
\end{equation*}
Moreover for $n$ large enough we have
\begin{align*}
\PP(\exists i,j \;\textrm{s.t.}\; i \bowtie j)
&\ge \PP(\exists \;\textrm{odd}\; i \le n-1\;\textrm{s.t.}\; i \bowtie (i+1))\\
&\ge 1-(1 - \PP(1 \bowtie 2))^{\lfloor (n-1)/2\rfloor}\\
&\ge 1-\biggl(1+\frac{n}{8}\exp\biggl({-\frac{11(C_1+C_2\sigma+\sqrt{s/b_n})^2sb_n}{2d\sigma^4}}\biggr)\biggr)^{-1}.
\end{align*}
We have assumed that $\frac{s^2}{2d\sigma^4}\ll b_n=O(\log(n))$ which implies that 
\begin{equation*}
\biggl(1+\frac{n}{8}
\exp\biggl({-\frac{11(C_1+C_2\sigma+\sqrt{s/b_n})^2sb_n}{2d\sigma^4}}\biggr)\biggr)^{-1}
\sim\frac{1}{4} \exp\biggl({-\frac{11s^2}{2d\sigma^4}}\biggr)=o(1).
\end{equation*}
Hence we obtain the desired result.
\end{proof}

\section{ADDITIONAL EXPERIMENTS}
We include experiments on two additional datasets:
Marvel Universe Social Graph (\url{https://syntagmatic.github.io/marvel/})
and PubMed Diabetes \citep{namata2012query}.
The Marvel dataset consists of Marvel characters (heroes) as nodes in the graph.
The node features are the comic issues in which the hero appeared.
Two heroes are connected by an edge if they appeared in the same comic issue.
This follows our definition of a random intersection graph.
The PubMed dataset is created similarly to Cora and CiteSeer,
but with more articles and fewer features.
Each publication is described by a TF/IDF weighted word vector.
A summary of the two datasets can be found in Table~\ref{tb:stats_add}.

\begin{table}[h!]
\centering
\caption{Summary of datasets.}
\label{tb:stats_add}
\begin{tabular}[t]{c|c|c|c}
\hline
Dataset & \# Vertices & \# Edges & \# Features\\ 
\hline
Marvel & $6,444$ & $574,467$ & $12,849$\\
PubMed & $19,717$ & $88,648$ & $500$\\
\hline
\end{tabular}
\end{table}

We again add independent Gaussian noise to the node features
and find the matching using the linear method and the GNN.
The results are shown in Figure~\ref{fg:acc_add}.
Similar trends can be observed in the plots:
the linear method fails when the feature noise is large,
while the GNN can tolerate larger amounts of noise
with the help of the graph structure.
\begin{figure}[h!]
\centering
\subfigure[Marvel]{
\includegraphics[width=0.45\textwidth]{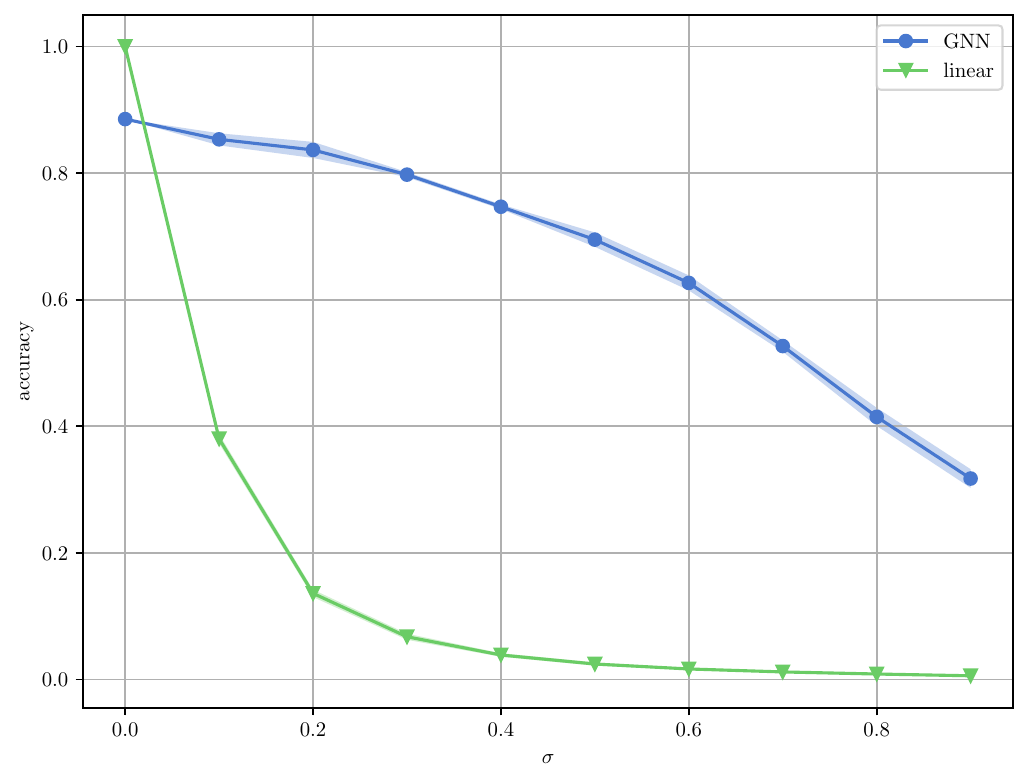}
}
\subfigure[PubMed]{
\includegraphics[width=0.45\textwidth]{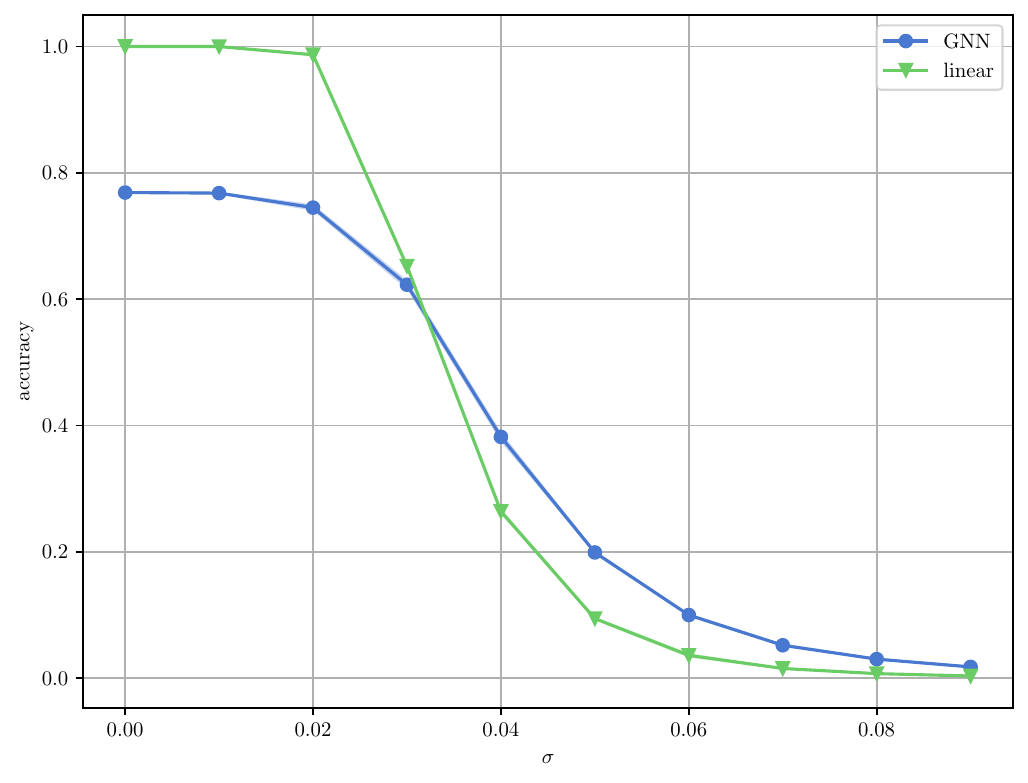}
}
\caption{Impact of the noise parameter on real-world datasets.}
\label{fg:acc_add}
\end{figure}

\end{document}